\documentclass[preprint,10pt,review]{elsarticle}
\usepackage{fontenc}
\usepackage{longtable}

\usepackage{placeins}
\usepackage{lineno}
\usepackage{graphicx}

\usepackage{algorithm}
\usepackage{algorithmicx}

\usepackage{amsmath}
\usepackage{multirow}
\usepackage{amsfonts}
\usepackage{mathtools} 
\usepackage{algpseudocode}

\usepackage{amssymb}
\usepackage{color}
\usepackage{amsmath} 

\usepackage{mathdots}
\usepackage{subcaption}
\usepackage{longtable}
\usepackage{setspace}

\usepackage{booktabs}
\usepackage{subcaption}

\usepackage{xcolor}
\definecolor{mydarkblue}{RGB}{0,0,139}

\usepackage{amsthm} 
\newtheorem{theorem}{Theorem}

\newtheorem{lemma}[theorem]{Lemma}
\newtheorem{corollary}[theorem]{Corollary}
\newtheorem{definition}{Definition}
\newtheorem{assumption}[theorem]{Assumption}
\newtheorem{example}{Example}

\usepackage{colortbl}
\usepackage{siunitx} 

\usepackage{multirow}

\usepackage{colortbl}

\usepackage{graphicx} 
\usepackage{longtable} 
\usepackage{multirow} 
\usepackage{array} 
\usepackage{tabularx}

\usepackage[normalem]{ulem}
\useunder{\uline}{\ul}{}

\journal{arXiv}
\begin{document}

\begin{frontmatter}
\title{Data Skeleton Learning: Scalable Active Clustering with Sparse Graph
Structures}

\author[a]{Wen-Bo Xie}  
\author[a]{Xun Fu\corref{cor1}}
\author[a]{Bin Chen}
\author[b]{Yan-Li Lee}
\author[a]{Tao Deng}
\author[a]{Tian Zou} 
\author[a]{Xin Wang}
\author[c]{Zhen Liu}
\author[d]{Jaideep Srivastava}

\address[a]{School of Computer Science and Software Engineering, Southwest Petroleum University, \\Chengdu 610500, People's Republic of China}
\address[b]{School of Computer and Software Engineering, Xihua University, \\Chengdu 610039, People's Republic of China}
\address[c]{Web Sciences Center, University of Electronic Science and Technology of China, Chengdu 611731, People's Republic of China.}
\address[d]{College of Science and Engineering, University of Minnesota, Minneapolis MN 55455, United States of America.}

\cortext[cor1]{Corresponding authors at: School of Computer Science and Software Engineering, Southwest Petroleum University, Chengdu 610500, China. E-mail: fuxun0529@163.com (Xun Fu).
}

\begin{abstract}
In this work, we focus on the efficiency and scalability of pairwise constraint-based active clustering, crucial for processing large-scale data in applications such as data mining, knowledge annotation, and AI model pre-training. Our goals are threefold: (1) to reduce computational costs for iterative clustering updates; (2) to enhance the impact of user-provided constraints to minimize annotation requirements for precise clustering; and (3) to cut down memory usage in practical deployments. To achieve these aims, we propose a graph-based active clustering algorithm that utilizes two sparse graphs: one for representing relationships between data (our proposed data skeleton) and another for updating this data skeleton. These two graphs work in concert, enabling the refinement of connected subgraphs within the data skeleton to create nested clusters. Our empirical analysis confirms that the proposed algorithm consistently facilitates more accurate clustering with dramatically less input of user-provided constraints, and outperforms its counterparts in terms of computational performance and scalability, while maintaining robustness across various distance metrics.
\end{abstract}


\end{frontmatter}


\section{Introduction}

Clustering is a powerful tool for data mining, data pre-processing and knowledge annotation. Its use is therefore essential in advancing data-driven technologies and highly valuable across a range of applications, from industrial automation to intelligent systems \cite{ding2024clustering,gao2024improving}. Nevertheless, classical unsupervised clustering methods lack the finesse for high-quality groupings \cite{Xie2024Boosting,angell2022interactive}, highlighting the necessity of human-machine synergy. In light of this, constrained clustering using active learning frameworks (a.k.a. active clustering) \cite{deng2022query,zhou2023active} emerges as a crucial approach, with those based on pairwise constraints being particularly vital. This is because pairwise constraint-based active clustering algorithms adeptly handle the influx of data replete with ever-changing information. This capacity enables these algorithms to navigate the complexities of data such as voiceprints, industrial production data, and streaming videos, which do not conform to traditional labeling paradigms, thereby facilitating a more nuanced understanding and responding of information that spans a wide spectrum of contexts. 

However, the rapid expansion of data volumes presents new challenges to the practical applicability of these pairwise constraint-based active clustering algorithms. At the core of active clustering lies its human-in-the-loop design \cite{zhang2014context}. Typically, this loop starts with the machine initializing a preliminary clustering result. From there, the machine leads human-machine interactions by identifying suspicious data instances within the current clustering result and soliciting human feedback to refine the grouping \cite{xiong2013active}. This approach focuses on optimizing the clustering result with minimal user-provided constraints, contrasting with the less targeted methods of semi-supervised learning \cite{yang2020semi,cai2023review} and the exhaustive human-led clustering \cite{yang2020interactive,chen2021interactive}. Yet, this enhanced human-machine interactive efficiency incurs significant time and space costs \cite{shi2020fast,zhou2023active}: 
\begin{itemize}
\item The process of refining clustering results entails iteratively identifying suspicious data points and then reconstructing clusters, which results in a significant increase in time-complexity. 

\item The need for auxiliary structures substantially amplifies space-complexity. 

\item Dealing with pairwise constraints further exacerbates this challenge, resulting in a quadratic increase in both time and space complexities of the algorithm \cite{li2019ascent,deng2022query}. 
\end{itemize}
These bottlenecks in computational efficiency and scalability limit the practicality of active clustering based on pairwise constraints, particularly in the context of real-world big data applications.

To address this, we propose the \textbf{D}ata \textbf{S}keleton \textbf{L}earning-based Active Clustering (DSL), a scalable active learning framework for interactive clustering that incorporates a highly effective human-in-the-loop component. DSL excels in achieving more accurate clustering with minimal cost, leveraging a carefully crafted sparse graph (i.e., our proposed data skeleton) that captures intricate data relationships and nested clusters. This structure is further boosted by a machine-led, targeted learning loop that actively seeks human annotation on suspicious data. Core to this loop is a minimal constraint graph, tailored to minimize constraints burden at a low space cost. It only encapsulates constraints and is armed with a specialized deduction mechanism for efficient processing. This trifecta of innovations empowers DSL with low demands for user-provided constraints, reduced time-complexity, and minimal space-complexity, aligning with the practical requisites of large-scale data pre-processing. 

We outline our primary contributions as follows: 

(1) Introducing DSL: DSL hinges on two sparse graphs: our proposed data skeleton, for data relationships representation; and the minimal constraint graph, for efficient clustering update, which is guided by a distinctive deduction mechanism and user-provided constraints. (Section 4)

(2) Theoretical Assurance: It is proven that DSL achieves accurate clustering with a definitive user-provided constraints upper bound. Meanwhile, our analysis reveals that the DSL algorithm achieves a time complexity of $O(n^2\log n)$ and space complexity of $O(n)$, markedly surpassing its counterparts. (Section 5)

(3) Interactive Efficiency and Scalability: Demonstrated through 18 datasets, DSL achieves better clustering (measured by Adjusted Rand Index) with minimal user-provided constraints in most cases and showcases scalability up to half a million entries on a standard PC. Moreover, DSL outperforms its contemporaries in terms of response times as well as memory footprints on various-size synthetic datasets. (Section 6)

(4) Metric Robustness: DSL exhibits low sensitivity to different distance metrics, further highlighting its versatility and practicality. (Section 7)

\section{Related Work}

Previous research in the field of active clustering has concentrated on the identification of key data and the potential for this to be used to refine the clustering results.

\noindent\textbf{Key Data Identification.} Methods focusing on key data identification can be divided into the sample-based method and the sample-pair-based method \cite{xiong2016active}.

(1) Sample-based methods \cite{xiong2016active,fernandes2020improving} select data points of interest, then query pairwise constraints based on the selected key data points. Centrality-based criterion methods \cite{basu2004active,yu2017active} focus on 
exploring the distribution of clusters. For instance, FFQS \cite{basu2004active} introduces a farthest-first traversal scheme to select data points from different regions, thereby increasing the diversity among the chosen points. However, only utilizing the centrality criterion is far from enough \cite{shi2020fast}, due to the inability to fine-tune fuzzy areas between clusters. Therefore, uncertainty-based criterion \cite{xiong2013active,mallapragada2008active} were proposed. 
MinMax \cite{mallapragada2008active} computes the maximum similarity between each data point and others in its neighborhood set \cite{basu2004active}, then selects the pair with the smallest maximum similarity, representing the pair with the highest uncertainty.
Additionally, multiple-based criteria \cite{shi2020fast,li2019ascent,hazratgholizadeh2023active} were proposed to combine diversified criteria to identify more informative data points. 
One of the most representative algorithms is ADP \cite{shi2020fast}. ADP utilizes the density peak clustering method \cite{rodriguez2014clustering} for initial clustering. It then queries users for constraints on density peak-based cluster centers. Based on this, ADP applies Shannon’s theorem \cite{salazar2023phase} to assess the uncertainty of data points and queries additional constraints accordingly.

(2) Sample-pair-based methods \cite{sun2022active,abin2017random,abin2018querying} streamline the process by directly analyzing individual pairs of data points when querying pairwise constraints. However, when considering all possible pairs, it becomes necessary to determine the priority of each pair, which can limit the scalability of these methods. Density-based methods \cite{abin2020density,shen2023semi} offer a more cost-effective solution by capturing key relationships between point pairs based on regions or data structures identified through density estimation. For example, AAVV \cite{abin2020density} selects constraints that reveal cluster boundaries and intra-cluster structures using an auxiliary density-based structure, built by linking each point to its nearest denser neighbor. SSEHCCI \cite{shen2023semi} focuses on effectively delineating cluster boundaries and morphology by selecting constraints from points at dense cluster edges or between clusters. 

Our proposed DSL follows a sample-pair-based framework. Unlike previous methods, DSL identifies cluster centers (representative nodes in connected components) and outliers by prioritizing edge weights in descending order on a tailored sparse graph, i.e., the proposed data skeleton. This graph is iteratively constructed through the identification of high-density regions using a nearest neighbor approach. By adopting this methodology, DSL not only ensures computational efficiency but also demonstrates superior scalability when applied to large datasets, while minimizing memory usage.

\noindent\textbf{Clustering Results Refinement}. The focus of this type of method is on the refinement of clustering results. This process aims to enhance the quality of the clustering results to the greatest extent possible while minimizing the amount of user-provided constraints.

(1) Neighborhood set-based methods \cite{mallapragada2008active, xiong2013active, hazratgholizadeh2023active} are traditional strategies for this task. They initially collect constraints to generate a constraint graph and subsequently produce additional constraints leveraging the transitive properties of must-link and cannot-link constraints \cite{basu2008constrained}. 
However, as the number of constraints increases, both the memory and time costs tend to rise significantly. In order to facilitate the observation and adjustment of clustering results in real-time, online frameworks \cite{sun2022active,fernandes2020improving,deng2022query} have been proposed. In the context of online methods, the process of clustering is executed in a iterative manner, with each round selecting data that has been identified as suspicious based on the outcomes of previous rounds. However, this exacerbates the computational cost of the algorithm.


(2) Matrix analysis-based methods offer a distinctive approach to the utilisation of constraints. Active spectral clustering methods, such as ASC \cite{wang2010active} and IU-Red \cite{wauthier2012active}, utilize perturbations in the Laplacian matrix and analyze its eigenvectors to both select key data and refine the clustering results. However, they are highly sensitive to parameter selection. As a solution, ACSCHM \cite{wang2020active} employs the Hessian matrix to enhance its ability to capture the intrinsic geometric structure of datasets. Additionally, SPACE \cite{zhou2023active} refines clustering output by integrating multiple weak clustering results into a consensus matrix via self-paced learning. 

(3) In addition, hierarchical methods \cite{van2018cobra,van2018cobras} can also collaborate with interactive frameworks to improve the quality of the clustering by systematically fine-tuning at various levels of the cluster tree. COBRA \cite{van2018cobra} first generates multiple super-instances through k-means clustering. It then employs a bottom-up approach to refine the clustering result by aggregating these super-instances using must-link constraints. In contrast, COBRAS \cite{van2018cobras} 
strategically incorporates both must-link and cannot-link constraints to break down and reassemble clusters from the top-down using super-instances.

Our proposed DSL algorithm also adopts a top-down framework that incorporates both must-link and cannot-link constraints, while it introduces a more flexible mechanism for refining its tailored tree clustering results (i.e., the proposed data skeleton). By prioritizing larger edges, which can appear near both the root and leaf nodes, DSL is capable of effectively delineating large clusters while also identifying small isolated clusters, leading to more efficient clustering refinement. In addition, it is important to note that in DSL, `neighbors' refer to spatially adjacent nodes, which differs from the `neighborhood' concept used in neighborhood set-based methods.

\begin{table}[h]
\centering
\caption{{Time and space complexities of the mentioned methods. $\Delta$ presents the upper bound of the input user-provided constraints, and $n$ presents the scale of the data set. $t$, $m$, $p$, and $N_s$ are specific parameters only used in their respective papers.}}
\label{tab:complexity}
\resizebox{1.0\textwidth}{!}{%
\begin{tabular}{lllllllll}
\toprule
\textbf{Method} & \textbf{Year} & \textbf{Time-cost} & \textbf{Space-cost} && \textbf{Method} & \textbf{Year} & \textbf{Time-cost} & \textbf{Space-cost} \\
\cline{1-4} \cline{6-9}
FFQS \cite{basu2004active} & 2004 & $O(\Delta^2 n^2)$ & $O(n+\Delta^2)$ &&
MinMax \cite{mallapragada2008active} & 2008 & $O(\Delta^2 n^2)$ & $O(n+\Delta^2)$ \\
ASC \cite{wang2010active} & 2010 & $O(\Delta n^3)$ & $O(n^2)$ &&
IU-Red \cite{wauthier2012active} & 2012 & $O(\Delta n^3)$ & $O(n^2)$ \\
NPU \cite{xiong2013active} & 2013 & $O(\Delta^2 t n \log n)$ & $O(tn+\Delta^2)$ &&
URASC \cite{xiong2016active} & 2016 & $O(\Delta n^4)$ & $O(n^2)$ \\
ALSD \cite{yu2017active} & 2017 & $O(\Delta^2 n^2)$ & $O(n+\Delta^2)$ &&
RWACS \cite{abin2017random} & 2017 & $O(\Delta n^3)$  & $O(n^2)$ \\
COBRA \cite{van2018cobra} & 2017 & $O(n N_s + N_s^2 \log N_s)$ & $O(n)$ &&
COBRAS \cite{van2018cobras} & 2018 & $O(\Delta^2 \log \Delta)$ & $O(n)$ \\
QBC-FLA \cite{abin2018querying} & 2018 & Non-polynomial &  $O(n+\Delta^2)$ &&
ASCENT \cite{li2019ascent} & 2019 & Non-polynomial  & $O(n+\Delta^2)$ \\
ACSCHM \cite{wang2020active} & 2019 & $O(\Delta n^3)$ & $O(n^2)$ &&
AAVV \cite{abin2020density} & 2020 & $O(\Delta n^2)$ & $O(n^2)$ \\
FIECE-EM \cite{fernandes2020improving} & 2020 & $O(n^3)$ & $O(n)$ &&
ADP \cite{shi2020fast} & 2021 & $O(\Delta n^2)$ & $O(n^2)$ \\
QAAML \cite{deng2022query} & 2022 & $O(\Delta^2 t n^2 \log n)$ & $O(n^2)$ &&
ACDEC \cite{hazratgholizadeh2023active} & 2022 & Non-polynomial  & O(n)  \\
ADC \cite{sun2022active} & 2022 & Non-polynomial & O(n) &&
SSEHCCI \cite{shen2023semi} & 2023 & $O(p n^2 \log n)$ & $O(p n)$ \\
SPACE \cite{zhou2023active} & 2023 & Non-polynomial & $O(m n^2)$ &&
DSL  & Proposed & $O(n(\log n+ \Delta\log\Delta))$ & $O(n)$ \\
\bottomrule
\end{tabular}%
}
\end{table}

\noindent \textbf{In real-world applications}, however, the complexity of algorithms significantly affects the response time of interactive systems. This has prompted discussions about the practical value of certain algorithms, particularly those requiring real-time responses in online settings. Table 1 outlines the time and space complexities of the algorithms discussed above. It is evident that some algorithms prioritize reducing the demand for user-provided constraints, yet fail to address the implications of response delays caused by high computational complexity. In contrast, our proposed DSL algorithm makes innovative contributions to both key data identification and the refinement of clustering results. By utilising two sparse graph data structures, DSL significantly reduces the time cost while also decreasing the memory required.

\section{Preliminary}
We first clarify key concepts that will recur in later sections.

\begin{definition}
\label{Data Skeleton}
{\bf Data Skeleton.} A data skeleton is a sparse weighted directed graph designed to encapsulate the affiliations or relationships among data points. Given a dataset $\mathcal{X}=\{x_i\}_1^n$, its data skeleton is formally defined by a triplet $\mathcal{G}_s=(\mathcal{V}_s, \mathcal{E}_s, \mathcal{C}_s)$. Here, $\mathcal{V}_s$ represents the node set, corresponding to all the data points in $\mathcal{X}$. The edge set $\mathcal{E}_s$ comprises elements of the form $e_{ij} = (\langle x_i,x_j \rangle, w_{ij})$, where each edge not only indicates that $x_i$ should belong to the same cluster as $x_j$ but also carries a weight $w_{ij}$ denoting the distance {\rm $\text{dist}(x_i,x_j)$} between the two data points. Importantly, the total number of edges is constrained by $|\mathcal{E}_s|\leq|\mathcal{V}_s|$. $\mathcal{C}_s$ is the set of representative nodes for connected components in the data skeleton. 
\end{definition}

\begin{definition}
\label{Pairwise Constraints}
{\bf Pairwise Constraints \cite{wagstaff2001constrained}.} A pairwise constraint for nodes $x_i$ and $x_j$ is defined as $\tau_{ij} = ((x_i, x_j), \theta_{ij})$, with $\theta_{ij}$ being a binary indicator. 
{\bf NOTE:} In this paper, $\theta_{ij} \!=\! 0$ dictates a must-link, grouping $x_i$ and $x_j$ together, while $\theta_{ij} \!=\! 1$ enforces a cannot-link, keeping $x_i$ and $x_j$ separate. This counterintuitive notation is crucial for the mechanisms of our algorithm.
\end{definition}

\begin{definition}
\label{Constraint Graph}
{\bf Minimal Constraint Graph.} Given a data skeleton $\mathcal{G}_s$, a minimal constraint graph, denoted as $\mathcal{G}_{c}=(\mathcal{V}_c, \mathcal{E}_c)$,  is an weighted undirected graph composed exclusively of user-provided constraints, without including any constraints inferred through deduction. Herein, $\mathcal{V}_c \subseteq \mathcal{V}_s$ signifies a subset of nodes from $\mathcal{G}_s$, while $\mathcal{E}_c$, composed of elements $\{\tau_{ij}\}$, represents the set of pairwise constraints associated with $\mathcal{G}_s$.
\end{definition}

\begin{definition}
\label{Reciprocal Nearest Nodes} 
{\bf Reciprocal Nearest Neighbors (RNNs) \cite{xie2023PRSC}.} A pair of nodes $x_i$ and $x_j$ are reciprocal nearest neighbors if they are each other's closest nodes, denoted as $x_i= \mathcal{N}(x_j)$ and $x_j = \mathcal{N}(x_i)$.
\end{definition}

\noindent{\bf Other Notations:} In a data skeleton, the symbols with $s$ in the subscript are related to the data skeleton. $\mathcal{G}_s$ represents the Data Skeleton; $\mathcal{V}_s$, $\mathcal{E}_s$, $\mathcal{C}_s$ represent the node set, edge set, and representative node set of the data skeleton. The symbols with $c$ in the subscript are related to the minimal constraint graph. $\mathcal{G}_c$ represents the minimal constraint graph; $\mathcal{V}_c$ and $\mathcal{E}_c$ represent the node set and edge set of the minimal constraint graph. Given two data points $x_i$ and $x_j$, the cohesion probability $\textbf{P}^C_{ij}$ represents the likelihood of $x_i$ being in the same category as $x_j$. If $x_i$ and $x_j$ are already assigned to a cluster, the suspicion probability $\textbf{P}^S_{ij}$ denotes the likelihood that $x_i$ and $x_j$ are incorrectly assigned to the same cluster. Notations that will be used recurrently are listed in Table \ref{tab:key_symbol}.

\begin{table}
\centering
\caption{A summary of notations}
\label{tab:key_symbol}
\resizebox{\textwidth}{!}{%
\begin{tabular}{lllll}
\toprule
\textbf{Symbol} & \textbf{Explanation} &  & \textbf{Symbol} & \textbf{Explanation} \\
\midrule
$\mathcal{X}=\{x_i\}_1^n$ & Dataset with $n$ points && $\mathcal{G}_s$ & Data Skeleton \\
$\mathcal{G}_c$ & Minimal Constraint Graph && $\mathcal{V}_s$ & Node set in $\mathcal{G}_s$\\
$\mathcal{V}_c$ & Node set in $\mathcal{G}_c$&& $\mathcal{E}_s$ & Edge set in $\mathcal{G}_s$\\
$\mathcal{E}_c$ & Edge set in $\mathcal{G}_c$&& $\mathcal{C}_s$ & Representative node set in $\mathcal{G}_s$ \\
$e_{ij}$ & Edge on $\mathcal{G}_s$&& $w_{ij}$ & Distance of $x_i$ and $x_j$ on $e_{ij}$ \\
$\tau_{ij}$ & Pairwise constraint of $x_i$ and $x_j$ && $\theta_{ij}$ & Constraint type to be determined \\
$\mathcal{N}(x_i)$ & Nearest neighbor of $x_i$ && ${\rm\textbf{P}}^C_{ij}$ & Cohesion probability of $x_i$ and $x_j$ \\
$\text{dist}(x_i,x_j)$ & Distance of $x_i$ and $x_j$ && ${\rm\textbf{P}}^S_{ij}$ & Suspicion probability of $x_i$ and $x_j$ \\
$d_{in}(x_i)$ & In-degree of $x_i$ in $\mathcal{G}_s$ && $\mathcal{P}_{ij}$ & Shortest path of $x_i$ and $x_j$ in $\mathcal{G}_c$ \\
$k$ & Ground-truth categories in $\mathcal{X}$ && $\lambda$ &  Probability of erroneous edge\\
\bottomrule
\end{tabular}%
}
\end{table}

\section{Data Skeleton Learning}

DSL prioritizes the selection of pairwise data nodes with the highest suspicion for human annotation, maximizing the efficacy of each interactive update to the data structure. To begin with, we provide an overview of DSL's framework.

\subsection{Framework}

\begin{figure}[t]
\begin{center}
\centerline{\includegraphics[width=1.0\columnwidth]{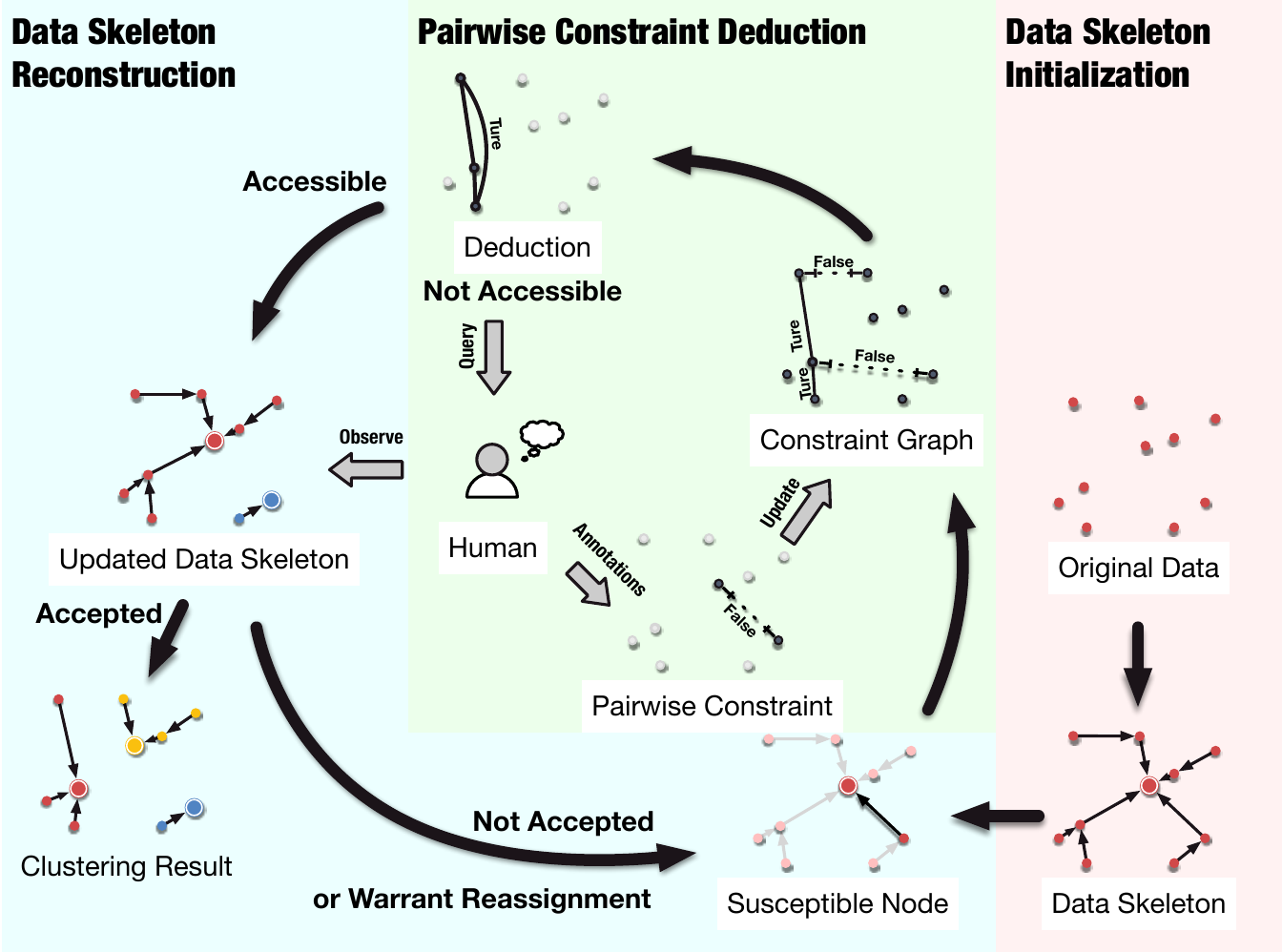}}
\caption{Schematic diagram of Data Skeleton Learning.}
\label{fig:framework}
\end{center}
\end{figure}

The proposed DSL is on the basis of the following assumption.

\begin{assumption}
Given two nodes $x_i$ and $x_j$, the node distance {\rm $\text{dist}(x_i,x_j)$} inversely corresponds to the cohesion probability {\rm$\textbf{P}^C_{ij}$}, which represents the likelihood of $x_i$ being in the same category as $x_j$. Conversely, if $x_i$ and $x_j$ are assigned to the same cluster, the distance directly relates to their suspicion probability {\rm$ \textbf{P}^S_{ij}$}, which indicates the chance of misassignment. Therefore: {\rm$\text{dist}(x_i,x_j)\propto {\textbf{P}}^S_{ij} \propto 1/\textbf{P}^C_{ij}$}. 
\end{assumption}

Building on the foundational assumption outlined above, it is evident that a naive approach to data interaction, predicated solely on pairwise distances, is markedly inefficient. This inefficiency stems from the fact that a dataset with $n$ items yields a total of $\frac{n(n-1)}{2}$ pairwise distances. Therefore, our proposed DSL leverages a tailored clustering tree, i.e., the proposed \textit{data skeleton}, to significantly reduce the volume of pairwise distances that need to be addressed. Based on this, DSL ensures the efficient optimization of clustering outcomes through the implementation of three main components: Data Skeleton Initialization (\textsc{DSInit}), Reconstruction (\textsc{Recons}), and Paiwise Constraint Deduction (\textsc{Deduction}). 

Figure \ref{fig:framework} provides a visual schematic of its design. Given a dataset $\mathcal{X}$, the process commences by deciphering inherent data relationships, and translating raw data into an initial data skeleton, denoted as $\mathcal{G}_s$ (\textsc{DSInit}). On this foundation, DSL launches a training loop, intrinsically intertwined with human interactions (\textsc{Recons}). Within this loop, DSL identifies nodes in the data skeleton that are most susceptible to allocation variances. Subsequently, leveraging the minimal constraint graph $\mathcal{G}_c$, constraints are either probed or deduced to ascertain if the node should uphold its initial cluster allocation or warrant reassignment (\textsc{Deduction}). In scenarios where the constraint graph yields ambiguous decisions, DSL queries user-provided constraints and subsequently refines the constraint graph. Depending on the deduction from the constraint graph (\textsc{Deduction}), the data skeleton is adjusted accordingly. The loop then revisits its origin, spotlighting another potential node with allocation ambiguities. As iterations progress, the reconstructed data skeleton embodies a nested clustering result. The loop ends when the user deems the clustering result satisfactory or when a pre-set number of rounds has been reached, at which point DSL outputs the corresponding clustering results based on the connected subgraphs in $\mathcal{G}_s$. The corresponding workflow is shown in Algorithm \ref{alg:framework}.

The subsequent subsections delve into the detailed intricacies and fundamental mechanics underpinning our DSL algorithm.

\begin{algorithm}[tb]
   \caption{\textsc{DSL}}
   \label{alg:framework}\label{alg:process}
   {\scriptsize
    \begin{algorithmic}[1]
        \Statex \textbf{Input}: Data: $\mathcal{X}$=$\{x_i\}_1^n$
        \Statex \textbf{Output}:  Labels: L
        \State{$\mathcal{G}_{s}$ $\gets$ $\textsc{DSInit}$($\mathcal{X}$)}  \hfill \textcolor{blue}{\scriptsize $\triangleright$ Algorithm \ref{alg:DSInit}}
        \State $\mathcal{G}_{c} \gets (\mathcal{V}_c = \varnothing,\mathcal{E}_c = \varnothing)$  
        \While{\textsc{Accept}($\mathcal{G}_{s}$) \textbf{is} False \textbf{and} loop count not exceeded}  
            \State $\mathcal{G}_{s}, \mathcal{G}_{c} \gets$ \textsc{Recons}($\mathcal{G}_{s}, \mathcal{G}_{c}$) \hfill \textcolor{blue}{\scriptsize $\triangleright$ Algorithm \ref{alg:DSRECons}}
        \EndWhile
        \State L $\gets$ \textsc{Label}($\mathcal{G}_{s}$) \hfill \textcolor{blue}{\scriptsize $\triangleright$ Based on the connected subgraphs in $\mathcal{G}_{s}$ }
        \State \textbf{return} $L$
    \end{algorithmic}
    }
\end{algorithm}

\subsection{Data Skeleton Initialization}

Our primary motivation for constructing the data skeleton within DSL is to dramatically reduce the number of potential pairwise node comparisons, even before human-machine interaction begins, ensuring more accurate and efficient targeting of nodes that require correct assignment in clustering. In light of this, we design the \textsc{DSInit} algorithm for initializing a data skeleton.

As illustrated in Algorithm \ref{alg:DSInit}, \textsc{DSInit} takes dataset $\mathcal{X}$ as input to construct the data skeleton $\mathcal{G}_{s}$ as its output. The algorithm starts by creating an empty $\mathcal{G}_{s}$, adding nodes from $\mathcal{X}$ to it, and treating all nodes in $\mathcal{G}_{s}$ as representative nodes, i.e., $\mathcal{C}_s = \mathcal{X}$ (line 1). It then undergoes an iterative construction on $\mathcal{G}_{s}$ (lines 2-14), where in each iteration, the focus is solely on the representative nodes. Each iteration primarily involves two key steps: (1) connecting each node $x_i$ in $\mathcal{C}_s$ to its nearest neighbor $\mathcal{N}_{\mathcal{C}_s}(x_i)$ in $\mathcal{C}_s$ (lines 3-5) and (2) updating the representative node set $\mathcal{C}_s$ based on the connection results (lines 6-13). \textsc{DSInit} first identifies all RNNs pairs in $\mathcal{C}_s$ and resets the representative node set (lines 6-7). Then, for each pair of RNNs, \textsc{DSInit} selects the node with the higher in-degree as the representative node and adds it to the updated representative node set, while removing the edge from the node with the higher in-degree to the node with the lower in-degree (lines 8-13); if in-degrees are equal, a node is chosen randomly. 
The iteration continues until only a single representative remains in $\mathcal{C}_s$.




Sorting out the \textsc{DSInit} algorithm, we can find that \textsc{DSInit} facilitates the strategic organization of the data skeleton,  which forms a tree-shaped graph. This single connected graph has $n-1$ edges, a unique root (the final representative node, which is the only node with no outgoing edges), and contains no cycles. Additionally, the data skeleton is a hierarchical, nested clustering result. It is formed by iteratively merging nodes (or clusters) and ensuring that the relationships between them are organized in a parent-child manner. Moreover, this tailored tree has a unique configuration, where edges with significant weights (indicating large distances between nodes) are positioned near either the root or the extremities of the clustering tree. This enables the algorithm to efficiently separate clusters or swiftly eliminate inaccuracies from clusters in subsequent stages.

\begin{algorithm}[t]
    \caption{\textsc{DSInit}} 
    \label{alg:DSInit}
    {\scriptsize
    \begin{algorithmic}[1] 
        \Statex \textbf{Input}: $\mathcal{X}$=$\{x_i\}_1^n$
        \Statex \textbf{Output}: Data Skeleton: $\mathcal{G}_{s}$
        \State $\mathcal{G}_{s} \gets (\mathcal{V}_s =$ $\mathcal{X}$$,\mathcal{E}_s = \varnothing,\mathcal{C}_s = \mathcal{X})$; 
        \While{$|\mathcal{C}_s|>1$} 
            \For{\textbf{each} $x_i$ \textbf{in} $\mathcal{C}_s$}
                \State$\mathcal{E}_{s} \gets \mathcal{E}_{s} \cup \{(\langle x_i, {\mathcal{N}_{\mathcal{C}_s}} (x_i)\rangle, {\rm{dist}}(x_i,{\mathcal{N}_{\mathcal{C}_s}}(x_i)))\}$
            \EndFor
            \State{RNNsSet$_{\mathcal{C}_s} = \{(x_i,x_j)|e_{ij}\in\mathcal{E}_{s} \land e_{ji}\in\mathcal{E}_{s}, \; x_i,x_j \in \mathcal{C}_s\}$}
            \State{$\mathcal{C}_s \gets \varnothing$}
            \For{\textbf{each} $(x_i,x_j)$ \textbf{in} {RNNsSet$_{\mathcal{C}_s}$}}
                \State $\mathcal{C}_s \gets \mathcal{C}_s \cup \{\arg\max\limits_{x_r \in \{x_i, x_j\}} d_{in}(x_r)\}$
                \State \textbf{if} {$d_{in}(x_i) > d_{in}(x_j)$} \textbf{then}
                 $\mathcal{E}_{s} \gets \mathcal{E}_{s}\backslash \{e_{ij}\}$
                \State \textbf{else} $\mathcal{E}_{s} \gets \mathcal{E}_{s} \backslash \{e_{ji}\}$
                \State \textbf{endif}
            \EndFor
        \EndWhile

    \State \textbf{return} $\mathcal{G}_{s}$
    \end{algorithmic}
    }
\end{algorithm}

\begin{example}
\label{example: DSInit}
\rm
As illustrated in Figure \ref{fig:DSInit}(a), we demonstrate the data skeleton initialization using \textsc{DSInit} with 10 nodes as an example. In the first iteration, each data point stands as a distinct connected subgraph, assuming the role of its representative. Subsequently, representatives are interconnected with their nearest neighbors, characterized by the edge weights reflecting the inter-node distances, as demonstrated in Figure \ref{fig:DSInit}(b). \textsc{DSInit} then isolates all reciprocal nearest node pairs. From these pairs, nodes with the maximum in-degree are selected as the respective connected subgraph's representatives. This selection yields three representative nodes, highlighted in additional circles in Figure \ref{fig:DSInit}(c).  As depicted in Figure \ref{fig:DSInit}(d), after successively connecting representatives to their nearest neighbors and determining new representatives in the subsequent iterations, the initialization of the data skeleton through \textsc{DSInit} is completed. \looseness=-1
\end{example}

\begin{figure}[h]
    \begin{center}
    \includegraphics[width=1.0\linewidth]{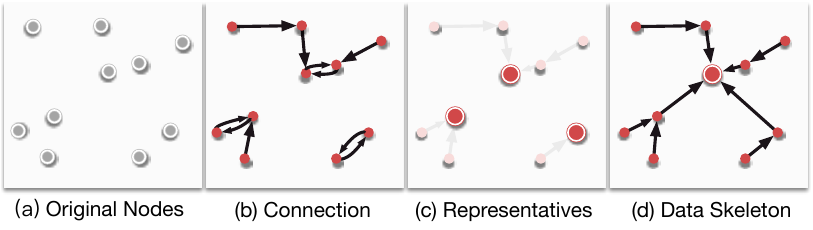}
    \caption{Data Skeleton Initialization}
    \label{fig:DSInit}
    \end{center}
\end{figure}

\subsection{Data Skeleton Reconstruction} 

Building upon the constructed data skeleton, the \textsc{Recons} algorithm is designed to refine the data skeleton to improve nested clustering results. It operates by selecting the edge with the greatest weight in $\mathcal{G}_s$ (indicating the longest distance between nodes) and treats the corresponding node pair as the most suspicious for potential separation. As we mentioned before, edges bearing significant weight are strategically positioned either near the root or the extremities of the clustering tree. This enables users to preferentially separate clusters or purge inaccuracies from clusters. Deductions are then made using the minimal constraint graph $\mathcal{G}_c$  (depending on Algorithm \ref{alg:PCDeduction}). If the deduction suggests the node pair should be disconnected, the source node of the edge is then matched with a more suitable target within the representative node set $\mathcal{C}_s$, determined again by deduction. This process may result in either the transfer of a cluster's branch to another cluster or the establishment of the source node as a new representative point, forming an independent nested cluster if no suitable target is found in $\mathcal{C}_s$.

\begin{algorithm}[tb]
    \caption{\textsc{Recons}}
    \label{alg:DSRECons}
    {\scriptsize
    \begin{algorithmic}[1]
        \Statex \textbf{Input}: Data Skeleton: $\mathcal{G}_{s}=(\mathcal{V}_{s},\mathcal{E}_{s},\mathcal{C}_{s})$, Minimal Constraint Graph: $\mathcal{G}_{c}=(\mathcal{V}_{c},\mathcal{E}_{c})$
        \Statex \textbf{Output}: Updated $\mathcal{G}_{s}$ and $\mathcal{G}_{c}$
        \State $e_{ij} \gets \arg\max_{e_{ij} \in \mathcal{E}_{s}}(w_{ij})$
        \State $\tau_{ij} \gets$ \textsc{Deduction}$(x_i, x_j, \mathcal{G}_{c})$ \hfill \textcolor{blue}{\scriptsize $\triangleright$ Algorithm {\ref{alg:PCDeduction}}}
        \If{$\theta_{ij}$ \textbf{is} 0}
        \State Update $\mathcal{E}_{s}$ with $w_{ij}\gets 0$ 
        \Else
            \State $\mathcal{E}_{s}\gets \mathcal{E}_{s}\backslash\{e_{ij}\}$;
            $\mathit{trigger}$ $\gets $ False

            \For{\textbf{each} $x_r \textbf{ in } \text{argsort}_{x_r \in \mathcal{C}_s}\text{dist}(x_i,x_r)$}
            
                \State $\tau_{ir} \gets$ \textsc{Deduction}$(x_i, x_r, \mathcal{G}_c)$  \hfill\textcolor{blue}{\scriptsize $\triangleright$ Algorithm {\ref{alg:PCDeduction}}}
                \If {$\theta_{ir}$ \textbf{is} 0}
                    \State $(x_r,x_i) \gets (\arg\max\limits_{x\in \{x_i,x_r\}} d_{in}(x),  \arg\min\limits_{x\in \{x_i,x_r\}} d_{in}(x))$  \hfill\textcolor{blue}{\scriptsize $\triangleright$ Reassign $x_r$,$x_i$}
                    \State Update $\mathcal{C}_{s}$ with $\mathcal{C}_s \cup \{x_r\} \backslash \{x_i\}$
                    \State Update $\mathcal{E}_{s}$ with $\{e_{ir}=(\langle x_i, x_r\rangle, w_{ir}\gets0)\}$
                    \State $\mathit{trigger}$ $\gets$ True; \textbf{break}
                \EndIf
            \EndFor
            \If {$\mathit{trigger}$ \textbf{is} False}
            \State $\mathcal{C}_s$ $\gets$ $\mathcal{C}_s$ $\cup$ $\{x_i\}$
            \EndIf
        \EndIf
        \State \textbf{return} $\mathcal{G}_{s}$, $\mathcal{G}_{c}$ 
    \end{algorithmic}
    }
\end{algorithm}

The algorithmic steps of the \textsc{Recons} algorithm are illustrated in Algorithm \ref{alg:DSRECons}. \textsc{Recons} processes the data skeleton $\mathcal{G}_{s}$, the minimal constraint graph $\mathcal{G}_{c}$, and the pairwise constraint $\tau_{ij}$ to produce an updated version of both $\mathcal{G}_{s}$ and $\mathcal{G}_{c}$.
\textsc{Recons} commences by pinpointing the node pair $(x_i,x_j)$ most susceptible to alterations (line 1). Drawing from the constraint graph, \textsc{Recons} deduces the respective pairwise constraints $\tau_{ij}$ (line 2).
The subsequent course of action taken by \textsc{Recons} hinges on the intrinsic character of this deduced constraint. Specifically, for a must-link constraint (lines 3-4), where $\theta_{ij}=0$, it updates the data skeleton $\mathcal{G}_{s}$ by assigning a zero weight to the edge between nodes $x_i$ and $x_j$, indicating that their initial association is apt and without suspicion. 
On the flip side, in the case of a cannot-link constraint (lines 5-19), \textsc{Recons} aims to find an appropriate node in $\mathcal{C}_s$ to connect with $x_i$. To this end, it deletes $x_i$'s current outgoing edge and introduces a $\mathit{trigger}$ variable to track if $x_i$ has successfully connected to a new node (line 6). 
Subsequently, the algorithm inspects nodes in $\mathcal{C}_s$ prioritizing those closer to $x_i$, until it identifies a fitting node (lines 7-16). During this traversal, \textsc{Recons} employs the \textsc{Deduction} function on $x_i$ to determine its relation with $x_r$ (line 8). When $\tau_{ir}$ corresponds to a must-link, i.e., $\theta_{ir}=0$, the algorithm delves into an updating phase (lines 9-15). However, if $\tau_{ir}$ suggests a cannot-link, the traversal persists. 
The updating phase involves several key actions: selecting the node with the maximum in-degree between $x_i$ and $x_r$ to act as the representative (lines 10-11); assigning a zero weight to the edge between $x_i, x_r$ and targeting the node with the higher in-degree (line 12); and flagging the $\mathit{trigger}$ as True, signifying a successful connection of $x_i$ to a representative in $\mathcal{C}_{s}$ (line 13).
If the connection attempt for $x_i$ fails, it gets designated as a new independent representative node (lines 16-17). Following the resolution of the above cases, the updated $\mathcal{G}_{s}$ and $\mathcal{G}_{c}$ are returned as output (line 20). Note that $\mathcal{G}_{c}$ is updated by \textsc{Deduction}.
To offer a deeper insight into the workings of \textsc{Recons}, we provide a visual demonstration in Example \ref{example: reconstruction}.

\begin{figure}[h]
    \centering
    \includegraphics[width=1.0\linewidth]{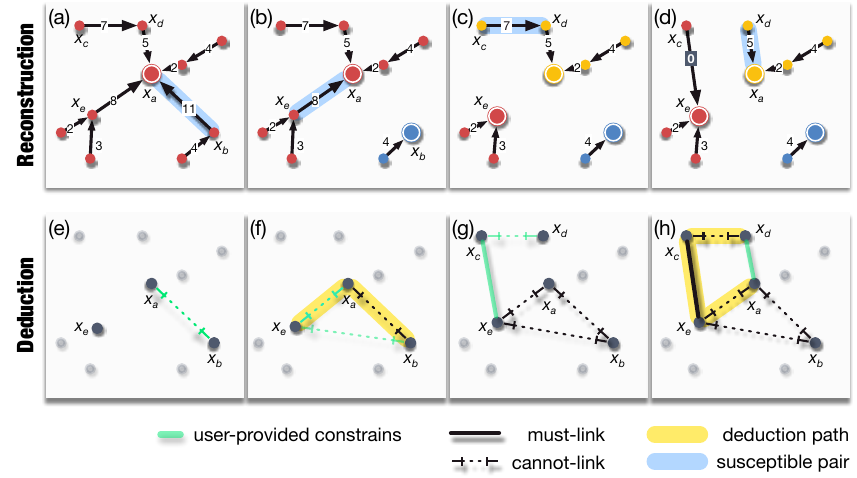}
    \caption{Constraint Deduction v.s. Data Skeleton Reconstruction}
    \label{fig:DSRECons}
\end{figure}

\begin{example}
\label{example: reconstruction}
\rm
Figure \ref{fig:DSRECons} presents the \textsc{Recons} applied to the data skeleton from Example \ref{example: DSInit} across four iterations. Nodes $x_b$, $x_e$, $x_c$, and $x_d$ are sequentially updated based on the edge weights, as illustrated in Figures \ref{fig:DSRECons}(a-d): $x_b$ emerged as a new representative after disconnecting its prior link; $x_e$ emerged as new representative as well; $x_c$ reattaches to the nearest representative node $x_e$; $x_d$ retains its link. Concurrent updates to the minimal constraint graph, as shown in Figures \ref{fig:DSRECons}(e-h), are driven by user-provided constraints. A detailed version of this example is provided in the supplementary materials.
\end{example}

\subsection{Pairwise Constraint Deduction}
We then delve into the \textsc{Deduction} algorithm, tailored for determining if two nodes belong to the same cluster. As depicted in Algorithm \ref{alg:PCDeduction}, \textsc{Deduction} is based on the proposed Theorem \ref{theorem:theorem1} (lines 1-5). If Theorem \ref{theorem:theorem1} cannot derive the constraint, human annotation is then requested (lines 6-9). \looseness-1

\begin{algorithm}[ht]
    \caption{\textsc{Deduction}} 
    \label{alg:PCDeduction}
    {\scriptsize
    \begin{algorithmic}[1] 
        \Statex \textbf{Input}: Nodes: $x_i,x_j$, Minimal Constraint Graph: $\mathcal{G}_{c}$=$(\mathcal{V}_{c},\mathcal{E}_{c})$
        \Statex \textbf{Output}: Pairwise Constraint: $\tau$
        \State $\tau_{ij} \gets ((x_i,x_j),\theta_{ij}=$\text{None})
        \State $\mathcal{P}_{ij}$ $\gets$ Find the shortest path between $x_i$ and $x_j$ in $\mathcal{G}_{c}$
        \If{$\mathcal{P}_{ij}$ \textbf{is not} None}
            \State Update $\tau_{ij}$ with Eq. (\ref{equ:equ3}) \hfill \textcolor{blue}{\scriptsize $\triangleright$ {Theorem \ref{theorem:theorem1}}}
        \EndIf
        \If{$\theta_{ij}$ \textbf{is} None}
        \State $\tau_{ij} \gets \textsc{Query}(x_i,x_j)$ \hfill \textcolor{blue}{\scriptsize $\triangleright$ \text{Send a human annotation request for the constraint}}
        \State{Update $\mathcal{G}_{c}$ with $\tau_{ij}$} 
        \EndIf
        \State \textbf{return} $\tau_{ij}$
    \end{algorithmic}
    }
\end{algorithm}

\begin{theorem}
\label{theorem:theorem1}
Given a minimal constraint graph $\mathcal{G}_{c}$ with nodes $x_i$ and $x_j$ for which the pairwise constraint $\tau_{ij} = ((x_i,x_j),\theta_{ij})$ needs to be inferred. By treating Boolean constraints as edge weights of 0 or 1, if the shortest path $\mathcal{P}_{ij}$ exists between them, the pairwise constraint can be deduced as follows:
{\rm
\scriptsize
\begin{equation}
\label{equ:equ3}
\text{$\theta_{ij}$} \gets \begin{cases}
    0 \text{(i.e., must-link)}, & |\mathcal{P}_{ij}| = 0 \\
    1 \text{(i.e., cannot-link)}, & |\mathcal{P}_{ij}| = 1 \\
    \text{None}, & |\mathcal{P}_{ij}| > 1
\end{cases},
\end{equation}
}

\noindent where $|\mathcal{P}_{ij}|$ refers to the shortest path length, None represents an inability to deduce. 
\end{theorem}

\newcounter{tempcounter1}
\setcounter{tempcounter1}{\value{theorem}}

\noindent\textit{Proof.} To prove this theorem, we delve into the correctness of Equation (\ref{equ:equ3}), and the essentiality of the shortest path. First of all, we cite two pivotal Lemmas.

\begin{lemma}
Transitivity of Must-Link Constraints \cite{basu2008constrained}. Given three nodes $x_a$, $x_b$, and $x_c$, if there exists a must-link constraint between $x_a$ and $x_b$, and another must-link constraint between $x_b$ and $x_c$, then a must-link constraint is also inferred between $x_a$ and $x_c$, i.e.,
\begin{equation}
\scriptsize
    ((x_a,x_b),0)  \wedge ((x_b,x_c),0) \Rightarrow ((x_a,x_c),0)
\end{equation}
\end{lemma}

\begin{lemma}
Entailment of Cannot-Link Constraints \cite{basu2008constrained}. Given three nodes $x_a$, $x_b$, and $x_c$, if there is a must-link constraint between $x_a$ and $x_b$, but a cannot-link constraint between $x_b$ and $x_c$, then $x_a$ and $x_c$ must not be linked, i.e.,
\begin{equation}
\scriptsize
    ((x_a,x_b),0)  \wedge ((x_b,x_c),1) \Rightarrow ((x_a,x_c),1)
\end{equation}
\end{lemma}

On the basis of these two Lemmas, we provide two corollaries.

\begin{corollary}
\label{corollary:corollary1}
Given any deduction path composed of nodes $\{x_i, x_{i+1}, \cdots,$ $ x_{i+m-1}, x_{i+m}\}$, if every edge on the path is constrained by a must-link constraint, that is $|\mathcal{P}_{i,i+m}| = 0$, then we have:
{\scriptsize
\begin{align*}
    ((x_i, x_{i+1}), 0) \land ((x_{i+1}, x_{i+2}), 0) &\Rightarrow ((x_i, x_{i+2}), 0),\\
    ((x_i, x_{i+2}), 0) \land ((x_{i+2}, x_{i+3}), 0) &\Rightarrow ((x_i, x_{i+3}), 0),\\
    \cdots&\\
    ((x_i, x_{i+m-1}), 0) \land ((x_{i+m-1}, x_{i+m}), 0) &\Rightarrow ((x_i, x_{i+m}), 0),
\end{align*}
}

\noindent i.e., $\theta_{i,i+m}=0$. 
\end{corollary}

Corollary \ref{corollary:corollary1} implies that the transitivity of must-link constraints along the path allows us to infer a direct must-link constraint between the start and end nodes of the path.

\begin{corollary}
\label{corollary:corollary2}
Given any deduction path composed of nodes $\{x_i, x_{i+1}, \cdots, x_{i+m-1},$ $ x_{i+m}\}$, if there is only one edge with its two nodes $x_{i+k}, x_{i+k+1}$ on the path constrained by a cannot-link and the rest are must-link constraints, that is $|\mathcal{P}_{i,i+m}| = 1$, then we have:
{\scriptsize
\begin{align*}
    ((x_i, x_{i+1}), 0) \land ((x_{i+1}, x_{i+2}), 0) &\Rightarrow ((x_i, x_{i+2}), 0),\\
    \cdots&\\
    ((x_i, x_{i+k}), 0) \land ((x_{i+k}, x_{i+k+1}), 1) &\Rightarrow ((x_i, x_{i+k+1}), 1),\\
    ((x_i, x_{i+k+1}), 1) \land ((x_{i+k+1}, x_{i+k+2}), 0) &\Rightarrow ((x_i, x_{i+k+2}), 1),\\
    \cdots&\\
    ((x_i, x_{i+m-1}), 1) \land ((x_{i+m-1}, x_{i+m}), 0) &\Rightarrow ((x_i, x_{i+m}), 1),
\end{align*}
}

\noindent i.e., $\theta_{i,i+m}=1$. 
\end{corollary}

Corollary \ref{corollary:corollary2} indicates that the presence of a single cannot-link constraint along the path enforces a cannot-link relationship between the start and end nodes of the path. Thus, Corollaries \ref{corollary:corollary1} and \ref{corollary:corollary2} address the `0' and `1' cases, respectively. We next provide another corollary for the `None' case.

\begin{corollary}
Given any deduction path composed of nodes $\{x_i, x_{i+1}, \cdots, x_{i+m-1},$ $ x_{i+m}\}$, if the path contains multiple edges that are constrained by cannot-links, while the rest are constrained by must-links, that is $|\mathcal{P}_{i,i+m}| > 1$, then, according to Corollary 6, the path can always be transformed into one that contains only cannot-link constraints. This is because the initial deduction path can be divided into multiple sub-paths, each of which contains only one cannot-link constraint. In such a case, it is impossible to deduce the constraint between $x_i$ and $x_{i+m}$, i.e., $\theta_{i,i+m}=$ \rm{None}.  

\end{corollary}
\begin{proof}
\textit{We first explore the case when $|\mathcal{P}_{i,i+2}| = 2$, with nodes $\{x_i, x_{i+1}, x_{i+2}\}$. Here, the constraint between $x_i$ and  $x_{i+2}$ is unpredictable. This scenario can be understood analogously through the geometric relation of three collinear points on a plane. Building on this, we further extend our consideration to a longer path, denoted as $\{x_i, x_{i+1}, \cdots, x_{i+m-1}, x_{i+m}\}$, where $|\mathcal{P}_{i,i+m}| > 2$. If the relationship between $x_i$ and  $x_{i+2}$ is unpredictable, then by extension, the relationship between $x_i$ and  $x_{i+3}$ also remains indeterminable. By recursive extension, the relationship between $x_i$ and any subsequent node $x_{i+k}$ along the path also remains unpredictable. This cascading uncertainty in linkage predicates that beyond a certain point in the path, the predictability of connections between nodes cannot be ascertained when multiple cannot-link constraints are present. In another words, $|\mathcal{P}_{i,i+m}| > 1$ represents an inability to deduce, resulting in $\theta_{i,i+m}=\rm{None}$. }
\end{proof}

Thus, the correctness of Equation (\ref{equ:equ3}) is proven. Next, we begin the proof of the essentiality of the shortest path.

\begin{lemma}
Essentiality of Shortest Path. Given a constraint graph and nodes $x_i, x_{i+m}$ therein, only the shortest path connecting these two nodes is the only path that guarantees successful deduction of the constraints between them.
\end{lemma}

\begin{proof}
{\it
We proof this lemma by the Contrapositive. Firstly, the \textbf{Original Proposition} in Lemma 8 can be formally stated as: ``If there is a successful deduction of constraints between nodes $x_i, x_{i+m}$, then the path utilized for this deduction is the shortest path within the graph.'' In logical terms, this can be written as: If successful deduction ($S$), then shortest path ($Q$).

We next construct its \textbf{Contrapositive Proposition}, which is logically equivalent to the \textbf{Original Proposition}. The contrapositive is: ``If a path is not the shortest, then it will not lead to a successful deduction of the constraints between nodes $x_i, x_{i+m}$.'' In logical terms, this can be written as: If not shortest path ($\neg Q$), then not successful deduction ($\neg S$).

Now, we explain why this contrapositive holds.
The constraint graph is conflict-free. This is because, in active clustering methods, constraints are not arbitrarily added by users. Instead, they are introduced only when there is no deducible path between two nodes. Consequently, there cannot be two distinct paths between any two nodes, where one path has a weighted path length of 0 (the sum of the weights along the path) and the other has a weighted path length of 1. 
Therefore, if a path is not the shortest, its weighted path length must be greater than 1. This means that such a path cannot lead to a successful deduction of the constraints between the two nodes. Thus, the Contrapositive Proposition: If not shortest path ($\neg Q$), then not successful deduction ($\neg S$) is true.

Finally, by the principle of contrapositive logic (i.e., If $S$ then $Q$ $\Leftrightarrow$ If $\neg Q$ then $\neg S$), the Original Proposition must also hold true:
Only the shortest path connecting these two nodes guarantees the successful deduction of the constraints between them.
}
\end{proof}
The above constitutes the comprehensive proof of Theorem \ref{theorem:theorem1}. The correctness of Equation (\ref{equ:equ3}) has been demonstrated, and it is established that only the shortest path between two nodes guarantees the successful deduction of the constraints between them according to Equation (\ref{equ:equ3}).\hfill$\blacksquare$

Furthermore, illustrative examples of the deduction process are presented in Figure \ref{fig:DSRECons}.

\begin{example}
\label{example: deduction}
\rm

As depicted in Figure \ref{fig:DSRECons}(e), initially, there is no path between $x_e$ and $x_a$ in the minimal constraint graph, indicating that $|\mathcal{P}_{ea}|\!=\!\text{None}$. Therefore, human annotation is solicited, and the constraint between $x_e$ and $x_a$ is determined to be cannot-link, which is then added to the minimal constraint graph, as shown in Figure \ref{fig:DSRECons}(f). At this point, the shortest path $\mathcal{P}_{eb}$ consists of two cannot-link constraints: from $x_e$ to $x_a$ and from $x_a$ to $x_b$, yielding $|\mathcal{P}_{eb}|\!=\!2$. This leads to the  $\theta_{eb}\!=\!\text{None}$, prompting the need for further human annotation. Consequently, the constraint between $x_e$ to $x_a$ is determined to be cannot-link and is added to the constraint graph, as shown in Figure \ref{fig:DSRECons}(f). Additionally, as depicted in Figure \ref{fig:DSRECons}(g), the shortest path $\mathcal{P}_{da}$ consists of one must-link (from $x_c$ to $x_e$) and two cannot-link constraints (from $x_d$ to $x_c$ and from $x_e$ to $x_a$), resulting in $|\mathcal{P}_{da}|\!=\!2$. Therefore, the constraint $\tau_{da}$ necessitates human annotation, and the relationship between $x_d$ and $x_a$ is determined to be must-link and added to the constraint graph, as shown in Figure \ref{fig:DSRECons}(h).
\end{example}


Contrary to the counterparts, \textsc{Deduction} is meticulously designed to avoid altering the constraint graph based on deductive results, integrating only user-provided constraints. This design minimizes the computational cost of finding the shortest path, ensuring the efficiency of the deduction process.

\section{Analysis}
\label{sec:complexity}

\subsection{Theoretical Correctness}

Under ideal conditions (that is when data categories are objectively separable and users are error-free), our DSL algorithm is capable of achieving perfect clustering. A theoretical proof can be established based on a probabilistic assumption, i.e., the previously mentioned Assumption 1: $\text{dist}(x_i,x_j)\propto \textbf{P}^S_{ij} \propto 1/\textbf{P}^C_{ij}$. 

Leveraging this, \textsc{DSInit} iteratively constructs a data skeleton $\mathcal{G}_s$ with $n$ nodes as a problem of maximizing average cohesion probability. Each iteration connects current representative nodes to their nearest representatives to merge sub-clusters, optimizing the average cohesion probability during each iteration, formalized as:
\begin{equation}
\scriptsize
    \overline{\textbf{P}}_{\mathcal{C}_s}^C=\max  \sum_{x_i \in \mathcal{C}_s} \frac{\textbf{P}^C_{ij}}{|\mathcal{C}_s|}=\sum_{x_i \in \mathcal{C}_s} \max_{x_j \in \mathcal{C}_s} \frac{\textbf{P}^C_{ij}}{|\mathcal{C}_s|},
\end{equation}
where $x_j$ is also a member of $\mathcal{C}_s$ in the first equality. This guarantees the global optimality of new connections in every iteration.

As \textsc{DSInit} iterates, the cluster distances increase while the average cohesion probability decreases with each iteration, ultimately forming a hierarchical clustering tree.
The higher the hierarchy, the larger the distances, the lower the cohesion, and the higher the suspicion probability. 


We then calculate the overall suspicion probability across  $\mathcal{G}_s$, denoted as $\overline{\textbf{P}}^S$. $\overline{\textbf{P}}^S$  is defined as the average of the pairwise suspicion probabilities $\overline{\textbf{P}}^S_{ij}$, i.e.,  $\overline{\textbf{P}}^S=\text{ave}(\overline{\textbf{P}}^S_{ij})$, which reflects the overall level of suspicion between nodes in the graph.
With this basis, \textsc{Recons} refines this tree-shaped data skeleton (with $n\!-\!1$ edges) under human guidance, focusing on the most suspicious edges for restructuring. A node might link to a new representative or stand-alone, leading to a reduction of a non-zero element in $\overline{\textbf{P}}^S$, meanwhile, a decrease in the overall suspicion probability. After $n\!-\!1$ iterations, $\overline{\textbf{P}}^S$ certainly drops to zero, setting an upper bound of $n$ \textsc{Recons} loops for perfect clustering.

\subsection{Upper Bound on Required Constraints}

Our analysis leads to Theorem \ref{theorem:upper}, which establishes an upper bound of $(1+\lambda k)n$ on the number of pairwise constraints needed for DSL to achieve perfect clustering, where $k$ is the actual number of data categories, $\lambda$ is the probability of an erroneous edge (i.e., the two nodes should not be assigned in the same cluster) in the initial data skeleton. This theoretical underpinning guarantees the high interactive efficiency intrinsic to our algorithm.

\begin{theorem}
\label{theorem:upper}
For a dataset with $n$ items and $k$ ground-truth categories, where the distances between data points are well-defined and the probability of erroneous edge is $\lambda \in [0,1]$, the upper bound on the number of pairwise constraints that the DSL algorithm requires for fully accurate clustering (or the maximum number of edges in the constraint graph) is $(1+\lambda k)n$.
\end{theorem}

\begin{proof}
Pairwise data comparison, the simplest form of human annotation, is assumed to be error-free. Therefore, in the process of \textsc{Recons}, the data skeleton will not exceed $k$ representatives, as a new one emerges only when unrelated to any current representative. Once $k$ representatives are identified, they cover all categories, preventing additional representatives. If the original edge between two nodes is correct, a must-link is acquired. With the probability $\lambda$, this results in a total of $(1-\lambda)n$ constraints. Conversely, if the original edge between two nodes is erroneous, in the worst case during a \textsc{Recons} round, a target node links to the furthest representative. Consequently, this acquires a cannot-link with its original and $k-1$ others, plus a must-link with the furthest. In this case, the total number of constraints will not exceed $\lambda(k+1)n$. Combining the above situations, the total number of edges in the constraint graph will not exceed $(1+\lambda k)n$.
\end{proof}

\textbf{It is important to note} that Theorem \ref{theorem:upper} establishes a theoretical upper bound of the worst-case rather than the conditions necessary for achieving perfectly accurate clustering. In real applications where distance functions are well-defined, fully accurate clustering does not necessitate reducing $\overline{\textbf{P}}^S$ to zero. Furthermore, the farthest representative linkage scenario we mentioned is also atypical. Hence, as demonstrated in the Experiments section, our empirical results substantially surpass this theoretical upper bound.

\subsection{Complexity}

\noindent\textbf{Time-Complexity.} We begin by analyzing each of the three primary components:

(I) \textsc{DSInit}: Nearest neighbor linkage on multi-dimensional data across all $n$ points has a complexity of $O(n\log n)$. In the worst-known case \cite{xie2020RSC,xie2023PRSC}, the size of each merged cluster reduces by at least half. Thus, the overall complexity of this part is $T_1= n\log n+\frac{n}{2}\log \frac{n}{2}+\cdots \in O(n\log n)$.

(II) \textsc{Recons}: This phase begins with identifying the edge with the highest weight in the data skeleton, incurring a time complexity of $O(n)$. The subsequent several times of \textsc{Deduction} are crucial for its time complexity. In extreme cases, a node seeking reconnection may iterate through up to $k+1$ possible links, where $k$ is the count of ground-truth clusters, leading to a per-loop time-complexity for \textsc{Recons} of $T_2 = O((k+1)T_3)$.

(III) \textsc{Deduction}: The main computational load in this component is attributed to identifying the shortest paths between the target node pair, which is achieved through the shortest-path search algorithm. When the constraint graph encompasses $m$ nodes, the time complexity for \textsc{Deduction} is $T_3=O(m \log m)$.

We then consider the entire process across $\mu$ rounds of human-machine interaction. The overall time complexity of DSL is $T=T_1+\mu T_2 = T_1+\mu (k+1)T_3$. With a theoretical upper bound of $\Delta = (1+\lambda k)n$ constraints, DSL faces at most $n$ interaction rounds. Thus, we have $T\in O(n\log n)+O((k+1)n \Delta\log\Delta)$, which is of an upper bound $O(n(\log n+ \Delta\log\Delta))$, considering $k$ as a constant. To the best of our knowledge, this time-complexity represents the lowest time-complexity among existing counterparts that can guarantee fully accurate clustering.

\noindent\textbf{Space-Complexity.} The memory footprint of DSL is predominantly allocated to the data skeleton and the constraint graph. A data skeleton is composed of $n$ nodes and a maximum of $n$ edges, while the constraint graph contains no more than $n$ nodes and $(1+ \lambda k)n$ edges. Hence, DSL maintains an optimal space-complexity of $O(n)$ (considering $k$ as a constant), affirming its resource-efficient architecture.

\section{Experiments} 


\subsection{Experimental Setting }
\noindent\textbf{Datasets. \footnote{Source code, application demonstration, and datasets for DSL are available at: https://github.com/Traveler-fx/DSL}} We conduct experiments on 18 real-world datasets from the UCI Machine Learning Repository \cite{Dua2019UCI} to assess the performance of the proposed algorithm across a variety of domains. Table \ref{tab:dataset} provides the basic information of these datasets. In addition, the Olivetti Face Dataset \cite{samaria1994parameterisation} was included in the application of the algorithm to analyze face datasets.

\begin{table}[t]
\scriptsize
  \centering
  \caption{Basic information of the datasets. Note that the kddcup99 dataset used in this study is a subset of the original.}
  \label{tab:dataset}
  \resizebox{\linewidth}{!}{\begin{tabular}{@{}cccclcccc@{}}
    \toprule
    \textbf{Dataset}    & \textbf{\#Items}  & \textbf{\#Attribute} & \textbf{\#Class} && \textbf{Dataset}    & \textbf{\#Items}  & \textbf{\#Attribute} & \textbf{\#Class} \\
    \cline{1-4} \cline{6-9}
    tae     & 151   & 5       & 3     && led      & 500     & 7       & 10    \\
    divorce & 170   & 54      & 2     && balance  & 625     & 4       & 3     \\
    wine    & 178   & 13      & 3     && breast   & 699     & 9       & 2     \\
    thyroid & 215   & 5       & 3     && vehicle  & 846     & 18      & 4     \\
    ecoli   & 336   & 7       & 8     && banknote & 1,372    & 4       & 2     \\
    musk    & 476   & 166     & 2     && segment  & 2,310    & 19      & 7     \\
    \cline{1-4} \cline{6-9}
    EEG     & 14,980 & 14      & 2     && digits   & 70,000   & 784     & 10    \\
    letter  & 20,000 & 16      & 26    && skin     & 245,057  & 3       & 2     \\
    avlia   & 20,867 & 10      & 12    && kddcup99 & 489,843 & 38      & 23    \\ 
    \bottomrule
  \end{tabular}
  }
\end{table}

\noindent\textbf{Implementations.} For our experimental setup, a PC with a 3.9GHz CPU and 32GB of RAM running Windows 11 and Python 3.11 was used. For our benchmarks, we included three algorithms known for their high interaction efficiency: SPACE \cite{zhou2023active}, ADP \cite{shi2020fast}, and ADPE \cite{shi2020fast}. Additionally, we considered two scalable methods, COBRA \cite{van2018cobra} and COBRAS \cite{van2018cobras}, along with two classic algorithms, FFQS \cite{basu2004active} and MinMax \cite{mallapragada2008active}.

\noindent\textbf{Evaluation Indicator.} We first employed the Adjusted Rand Index \cite{ARI} (ARI) to benchmark the clustering accuracy with respect to the amount of user-provided constraints. 
Inspired by the ROC curve and its Area Under Curve (AUC) \cite{huang2005using}, we also introduced the Interactive Clustering Efficiency Curve (ICE Curve) and the Area Under the Interaction Curve (AUIC) as novel metrics to evaluate the overall efficiency of interactive clustering algorithms. 

The ICE Curve visually plots the progress in clustering accuracy against the cumulative amount of user-provided constraints. The AUIC metric, akin to the AUC in the ROC context, offers a numerical assessment of an algorithm's clustering process efficiency by accounting for both accuracy and the volume of constraints. We note the AUIC score at the dataset's size as AUIC@$n$, calculated by the formula: 
\begin{equation}
\scriptsize
    \text{AUIC}@n=\frac{1}{n}\sum_{i=0}^{n-1}\frac{s_i+s_{i+1}}{2},
\end{equation}
where $s_i$ denotes the clustering accuracy at the $i$-th constraint. 

A higher AUIC score signifies that the algorithm achieves better accuracy with fewer required constraints, indicating a more efficient approach to interactive clustering. An example of the ICE Curve and how it relates to AUIC is depicted in Figure \ref{fig:accI-intro}.

\begin{figure}[h]
    \centering
    \includegraphics[width=0.6\linewidth]{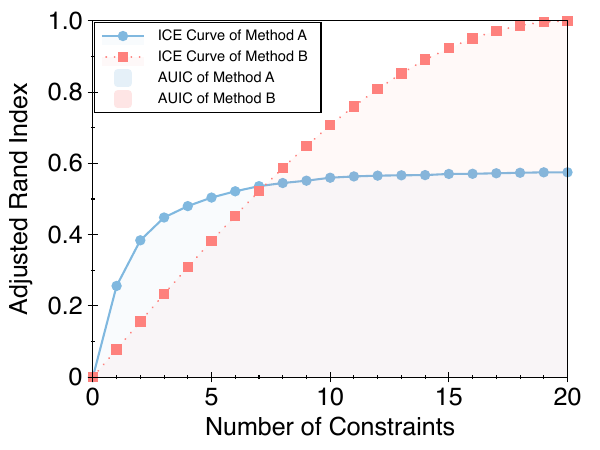}
    \caption{Example of the ICE Curve and AUIC. In this case, we set $n$ to 20. Here, Method A achieves an AUIC@20 of 0.510, whereas Method B attains an AUIC@20 of 0.636, showcasing Method B's superior efficiency in clustering.}
    \label{fig:accI-intro}
\end{figure}

\begin{table}[h]
\scriptsize
    \centering
    \caption{Comparative Analysis of Clustering Accuracy (ARI) v.s. User-provided Constraint  Counts on 12 small datasets. Highlighted values represent the highest ARI for a given constraint count.}
    \label{tab:accuracy}
    \resizebox{1\linewidth}{!}{%
    \setlength{\tabcolsep}{4mm}{
    \begin{tabular}{@{}lccccccccc@{}}
    \toprule
    \textbf{Dataset} & \textbf{\#Constraints} & \textbf{DSL} & \textbf{COBRA} & \textbf{COBRAS} & \textbf{ADP} & \textbf{ADPE} & \textbf{SPACE} & \textbf{FFQS} & \textbf{MinMax}\\
    \midrule
    \multirow{4}{*}{tae} & 130 & \cellcolor{yellow!25}\textbf{0.486} & 0.242 & 0.259 & 0.348 & 0.314 & 0.063 & 0.333 & 0.429\\
                         & 170 & \cellcolor{yellow!25}\textbf{0.738} & 0.521 & 0.365 & 0.526 & 0.457 & 0.063 & 0.686 & 0.435\\
                         & 210 & \cellcolor{yellow!25}\textbf{0.902} & 0.810 & 0.518 & 0.844 & 0.702 & 0.063 & 0.827 & 0.887\\  
                         & 250 & \cellcolor{green!15}\textbf{1.000} & 0.810 & 0.704 & \cellcolor{green!15}\textbf{1.000} & 0.940 & 0.022 & 0.892 & \cellcolor{green!15}\textbf{1.000}\\ \hline
    \multirow{4}{*}{divorce} & 20 & \cellcolor{yellow!25}\textbf{0.953} & 0.285 & 0.930 & \cellcolor{yellow!25}\textbf{0.953} & \cellcolor{yellow!25}\textbf{0.953} & 0.842 & 0.930 & 0.908\\
                             & 25 & \cellcolor{yellow!25}\textbf{0.976} & 0.310 & 0.930 & 0.953 & \cellcolor{yellow!25}\textbf{0.976} & 0.842 & 0.930 & 0.908\\
                             & 30 & \cellcolor{yellow!25}\textbf{0.976} & 0.334 & 0.953 & \cellcolor{yellow!25}\textbf{0.976} & \cellcolor{yellow!25}\textbf{0.976} & 0.842 & 0.930 & 0.953\\
                             & 35 & \cellcolor{green!15}\textbf{1.000} & 0.365 & 0.953 & \cellcolor{green!15}\textbf{1.000} & 0.976 & 0.842 & 0.930 & 0.976\\ \hline
    \multirow{4}{*}{wine}    & 60 & \cellcolor{yellow!25}\textbf{0.967} & 0.509 & 0.899 & 0.964 & 0.948 & 0.824 & 0.949 & 0.933  \\         
                                 & 70 & \cellcolor{yellow!25}\textbf{0.967} & 0.669 & 0.899 & 0.964 & 0.964 & 0.850 & 0.913 & 0.912  \\  
                                 & 80 & \cellcolor{yellow!25}\textbf{0.983} & 0.780 & 0.899 & 0.964 & 0.964 & 0.825 & 0.965 & \cellcolor{yellow!25}\textbf{0.983}   \\           
                                 & 90 & \cellcolor{green!15}\textbf{1.000} & 0.831 & 0.899 & 0.964 & 0.982 & 0.825 & 0.931 & 0.965   \\    \hline                   
    \multirow{4}{*}{thyroid} & 80 & \cellcolor{yellow!25}\textbf{0.921} & 0.884 & 0.833 & 0.000 & 0.000 & 0.628 & 0.748 & 0.889 \\
                             & 90 & \cellcolor{yellow!25}\textbf{0.921} & 0.899 & 0.833 & 0.000 & 0.000 & 0.628 & 0.874 & 0.840 \\
                             & 100 & \cellcolor{yellow!25}\textbf{0.952} & 0.925 & 0.892 & 0.000 & 0.000 & 0.628 & 0.905 & 0.873 \\
                             & 110 & \cellcolor{green!15}\textbf{1.000} & 0.954 & 0.892 & 0.000 & 0.000 & 0.628 & 0.906 & 0.880 \\ \hline
    \multirow{4}{*}{ecoli}   & 150 & \cellcolor{yellow!25}\textbf{0.845} & 0.804 & 0.712 & 0.498 & 0.690 & 0.423 & 0.461 & 0.464 \\
                             & 250 & \cellcolor{yellow!25}\textbf{0.897} & 0.804 & 0.737 & 0.895 & 0.856 & 0.422 & 0.620 & 0.864 \\
                             & 350 & 0.934 & 0.804 & 0.739 & 0.959 & \cellcolor{yellow!25}\textbf{0.973} & 0.436 & 0.859 & 0.934 \\ 
                             & 450 & \cellcolor{green!15}\textbf{1.000} & 0.804 & 0.816 & 0.996 & \cellcolor{green!15}\textbf{1.000} & 0.395 & 0.905 & 0.977 \\ \hline
    \multirow{4}{*}{musk}    & 350 & \cellcolor{yellow!25}\textbf{0.584} & 0.406 & 0.497 & 0.135 & 0.122 & -0.006 & 0.335 & 0.348 \\
                             & 450 & \cellcolor{yellow!25}\textbf{0.854} & 0.406 & 0.590 & 0.468 & 0.462 & -0.006 & 0.623 & 0.623 \\
                             & 550 & \cellcolor{yellow!25}\textbf{0.975} & 0.406 & 0.713 & 0.975 & 0.967 & -0.006 & 0.918 & 0.967 \\
                             & 650 & \cellcolor{green!15}\textbf{1.000} & 0.406 & 0.786 & \cellcolor{green!15}\textbf{1.000} & \cellcolor{green!15}\textbf{1.000} & 0.005 & \cellcolor{green!15}\textbf{1.000} & \cellcolor{green!15}\textbf{1.000} \\  \hline

    \multirow{4}{*}{led}     & 600 & \cellcolor{yellow!25}\textbf{0.745} & 0.526 & 0.526 & 0.647 & 0.420 & 0.513 & 0.727 & 0.555 \\
                             & 700 & \cellcolor{yellow!25}\textbf{0.832} & 0.526 & 0.526 & 0.707 & 0.575 & 0.519 & 0.826 & 0.735 \\
                             & 800 & \cellcolor{yellow!25}\textbf{0.910} & 0.526 & 0.526 & 0.759 & 0.750 & 0.561 & 0.827 & 0.908 \\
                             & 900 & \cellcolor{green!15}\textbf{1.000} & 0.526 & 0.526 & 0.882 & 0.948 & 0.597 & 0.964 & 0.973 \\ \hline
    \multirow{4}{*}{balance} & 250 & \cellcolor{yellow!25}\textbf{0.548} & 0.591 & 0.622 & 0.184 & 0.165 & 0.151 & 0.150 & 0.467 \\
                             & 450 & \cellcolor{yellow!25}\textbf{0.756} & 0.591 & 0.598 & 0.534 & 0.745 & 0.151 & 0.717 & 0.723 \\
                             & 650 & \cellcolor{yellow!25}\textbf{0.871} & 0.591 & 0.669 & 0.747 & 0.847 & 0.151 & 0.867 & 0.864 \\
                             & 850 & \cellcolor{green!15}\textbf{1.000} & 0.591 & 0.737 & \cellcolor{green!15}\textbf{1.000} & \cellcolor{green!15}\textbf{1.000} & 0.151 & \cellcolor{green!15}\textbf{1.000} & 0.989 \\ \hline
    \multirow{4}{*}{breast}  & 80  & \cellcolor{yellow!25}\textbf{0.899} & 0.766 & 0.861 & 0.871 & 0.882 & 0.837 & 0.888 & 0.870 \\
                             & 160 & \cellcolor{yellow!25}\textbf{0.915} & 0.872 & 0.695 & 0.893 & 0.954 & 0.837 & 0.850 & \cellcolor{yellow!25}\textbf{0.915} \\ 
                             & 240 & \cellcolor{yellow!25}\textbf{0.954} & 0.872 & 0.735 & 0.910 & 0.977 & 0.837 & 0.899 & 0.940 \\
                             & 320 & \cellcolor{green!15}\textbf{1.000} & 0.872 & 0.802 & 0.943 & 0.977 & 0.837 & 0.893 & 0.988 \\ \hline
    \multirow{4}{*}{vehicle} & 300 & \cellcolor{yellow!25}\textbf{0.328} & 0.373 & 0.249 & 0.188 & 0.136 & 0.077 & 0.087 & 0.142 \\
                             & 700 & \cellcolor{yellow!25}\textbf{0.608} & 0.373 & 0.472 & 0.533 & 0.459 & 0.153 & 0.381 & 0.434 \\
                             & 1,100& 0.833 & 0.373 & 0.603 & 0.889 & \cellcolor{yellow!25}\textbf{0.926} & 0.078 & 0.764 & 0.817 \\
                             & 1,500& \cellcolor{green!15}\textbf{1.000} & 0.373 & 0.757 & \cellcolor{green!15}\textbf{1.000} & \cellcolor{green!15}\textbf{1.000} & 0.078 & \cellcolor{green!15}\textbf{1.000} & \cellcolor{green!15}\textbf{1.000} \\  \hline
    \multirow{4}{*}{banknote}& 60  & \cellcolor{yellow!25}\textbf{0.968} & 0.179 & 0.873 & 0.000 & 0.770 & 0.000 & 0.035 & 0.034 \\
                             & 90  & \cellcolor{yellow!25}\textbf{0.997} & 0.836 & 0.957 & 0.000 & 0.957 & 0.000 & 0.003 & 0.043 \\
                             & 120 & \cellcolor{yellow!25}\textbf{0.997} & 0.994 & 0.957 & 0.000 & 0.931 & 0.000 & 0.049 & 0.054 \\
                             & 150 & \cellcolor{green!15}\textbf{1.000} & 0.994 & 0.985 & 0.000 & 0.994 & 0.000 & 0.057 & 0.078 \\ \hline
    
    \multirow{4}{*}{segment}  & 800  & \cellcolor{yellow!25}\textbf{0.917} & 0.770 & 0.870 & 0.318 & 0.540 & 0.470 & 0.552 & 0.532 \\
                              & 1,300 & \cellcolor{yellow!25}\textbf{0.972} & 0.770 & 0.938 & 0.744 & 0.773 & 0.470 & 0.724 & 0.693 \\
                              & 1,800 & \cellcolor{yellow!25}\textbf{0.995} & 0.770 & 0.942 & 0.803 & 0.851 & 0.469 & 0.890 & 0.750 \\
                              & 2,300 & \cellcolor{green!15}\textbf{1.000} & 0.770 & 0.956 & \cellcolor{green!15}\textbf{1.000} & 0.995 & 0.500 & \cellcolor{green!15}\textbf{1.000} & \cellcolor{green!15}\textbf{1.000} \\ 
    \bottomrule  
    \end{tabular}}
    }
\end{table}

\begin{table}[ht]
\centering
\caption{Comparative Analysis of Clustering Accuracy (ARI) v.s. Constraint Counts on 6 Large Datasets. Highlighted values represent the highest ARI for a given constraint count. `--' denotes cases where runtime exceeded 48 hours or memory capacity on our testbed. ADPE, SPACE, FFQS, and MinMax are not included in the table as they consistently exceeded 48 hours of runtime or exhausted memory across all 6 large datasets.}
\label{tab:accuracy-2}
\resizebox{\textwidth}{!}{
\begin{tabular}{cccccclcccccc}
\toprule
\textbf{Dataset} & \textbf{\#Constraints} & \textbf{DSL} & \textbf{COBRA} & \textbf{COBRAS} & \textbf{ADP} & & \textbf{Dataset} & \textbf{\#Constraints} & \textbf{DSL} & \textbf{COBRA} & \textbf{COBRAS} & \textbf{ADP} \\ 
\cline{1-6} \cline{8-13}
\multirow{4}{*}{EEG} & 4,000 & \cellcolor{yellow!25}0.302 & 0.059 & 0.231 & 0.015 & & \multirow{4}{*}{digits} & 27,000 & \cellcolor{yellow!25}0.887 & 0.750 & -- & -- \\
& 9,000 & \cellcolor{yellow!25}\textbf{0.544} & 0.059 & 0.394 & 0.248 &&& 42,000 & \cellcolor{yellow!25}\textbf{0.956} & 0.750 & -- & -- \\ 
& 14,000& \cellcolor{yellow!25}\textbf{0.754} & 0.059 & 0.530 & 0.751 &&& 57,000 & \cellcolor{yellow!25}\textbf{0.993} & 0.750 & -- &  -- \\
& 19,000& \cellcolor{green!15}\textbf{1.000} & 0.059 & 0.645 & \cellcolor{green!15}\textbf{1.000} &&& 72,000 & \cellcolor{green!15}\textbf{1.000}  & 0.750 & -- & --  \\ 
\hline
\multirow{4}{*}{letter} & 8,000 & \cellcolor{yellow!25}0.720 & 0.263 & 0.716 & -- && \multirow{4}{*}{skin} & 20,000 & \cellcolor{yellow!25}0.997 & 0.967 & -- & -- \\
& 16,000& \cellcolor{yellow!25}\textbf{0.912} & 0.263 & -- & -- &&& 50,000 & \cellcolor{yellow!25}\textbf{0.999} & 0.967 & -- & -- \\
& 24,000& \cellcolor{yellow!25}\textbf{0.977} & 0.263 & -- & -- &&& 80,000 & \cellcolor{yellow!25}\textbf{0.999} & 0.967 & -- & -- \\
& 32,000& \cellcolor{green!15}\textbf{1.000} & 0.263 & -- & -- &&& 110,000& \cellcolor{green!15}\textbf{1.000}  & 0.967 & -- & -- \\ 
\hline
\multirow{4}{*}{avila} & 19,000& \cellcolor{yellow!25}0.629 & 0.156 & -- & -- && \multirow{4}{*}{kddcup99} & 130  & \cellcolor{yellow!25}0.990 & 0.959 & -- & -- \\
& 27,000& \cellcolor{yellow!25}\textbf{0.776} & 0.156 & -- & -- &&& 1,300  & \cellcolor{yellow!25}\textbf{0.997} & 0.989 & -- & -- \\
& 35,000& \cellcolor{yellow!25}\textbf{0.897} & 0.156 & -- & -- &&& 13,000 & \cellcolor{yellow!25}\textbf{0.999} & 0.989 & -- & -- \\
& 43,000& \cellcolor{green!15}\textbf{1.000} & 0.156 & -- & -- &&& 130,000& \cellcolor{green!15}\textbf{1.000}  & 0.989 & -- & -- \\
\bottomrule
\end{tabular}
}
\end{table}

\subsection{Experimental Result}

\noindent\textbf{Exp. 1: Interactive Efficiency.} We first examined the impact of user-provided-constraint counts on the ARI score of various algorithms to validate the interactive efficiency of our algorithm. As shown in Tables \ref{tab:accuracy}-\ref{tab:accuracy-2}, we measured the Adjusted Rand Index (ARI) achieved by each method when provided with a fixed number of constraints. Our analysis spanned 12 small-scale datasets and 6 large-scale ones. The DSL algorithm consistently outperformed other benchmark methods: it achieved higher ARI values than competitors with an equal amount of constraints in all but two cases. Furthermore, it is observed from the last row of each dataset that our algorithm reached an ARI of 1 with fewer constraints than other algorithms that often did not achieve perfect grouping. This demonstrates that DSL can achieve 100\% accurate clustering results with the minimum constraints in most scenarios. It is worth noting that the complexity of the configuration of the SPACE, ADP, and ADPE parameters greatly affects their effectiveness. When the recommended parameters are extremely mismatched on some datasets, clustering will fail (e.g., the case of ARI<0). Another interesting observation is that the classic methods FFQS and MinMax sometimes outperform state-of-the-art methods. Nevertheless, these classic methods are significantly more time-consuming (see Figure \ref{fig:cpu_time} in Exp. 3).

\noindent\textbf{Exp. 2: Overall Evaluation.} 
Subsequently, we assessed the overall performance of interactive clustering algorithms using ICE Curves and AUIC scores. Figures \ref{fig:overall-data-1}-\ref{fig:overall-data-2} present the trends in clustering quality (ARI) as a function of the increasing volume of constraints. DSL algorithm demonstrates an overall advantage, consistently appearing above other curves in most cases. Table \ref{tab:ACCI-samll} details the comparative results of the AUIC@$n$ scores. As indicated, DSL algorithm secures the highest AUIC@$n$ score across most datasets, exhibiting superior performance with the highest average score (MEAN) as well. 

\begin{figure}[!h]
    \centering
    \includegraphics[width=0.9\linewidth]{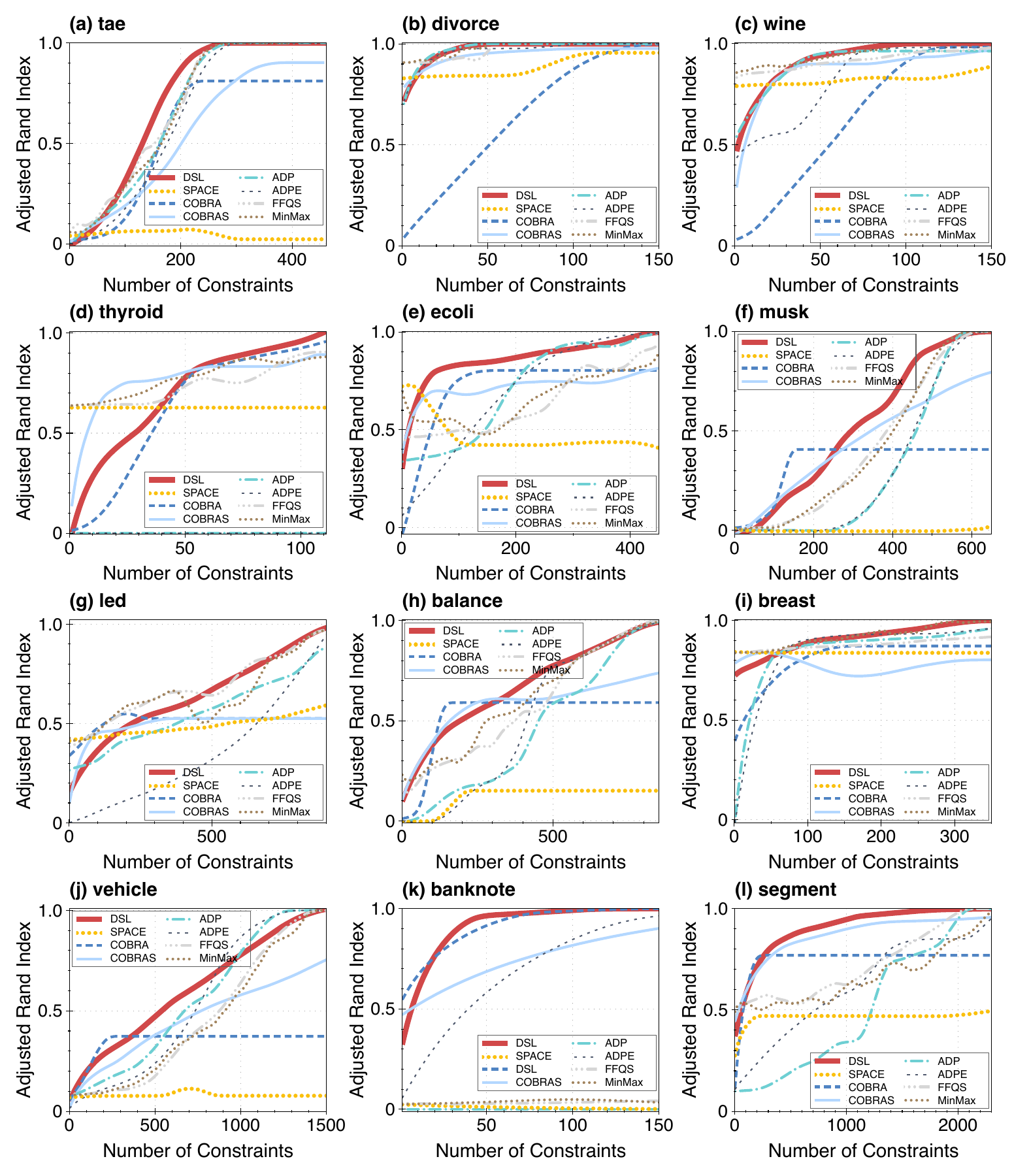}
    \caption{Comparison of ICE Curves on Small Datasets.}
    \label{fig:overall-data-1}
\end{figure}

\begin{figure}[!h]
    \centering
    \includegraphics[width=0.9\linewidth]{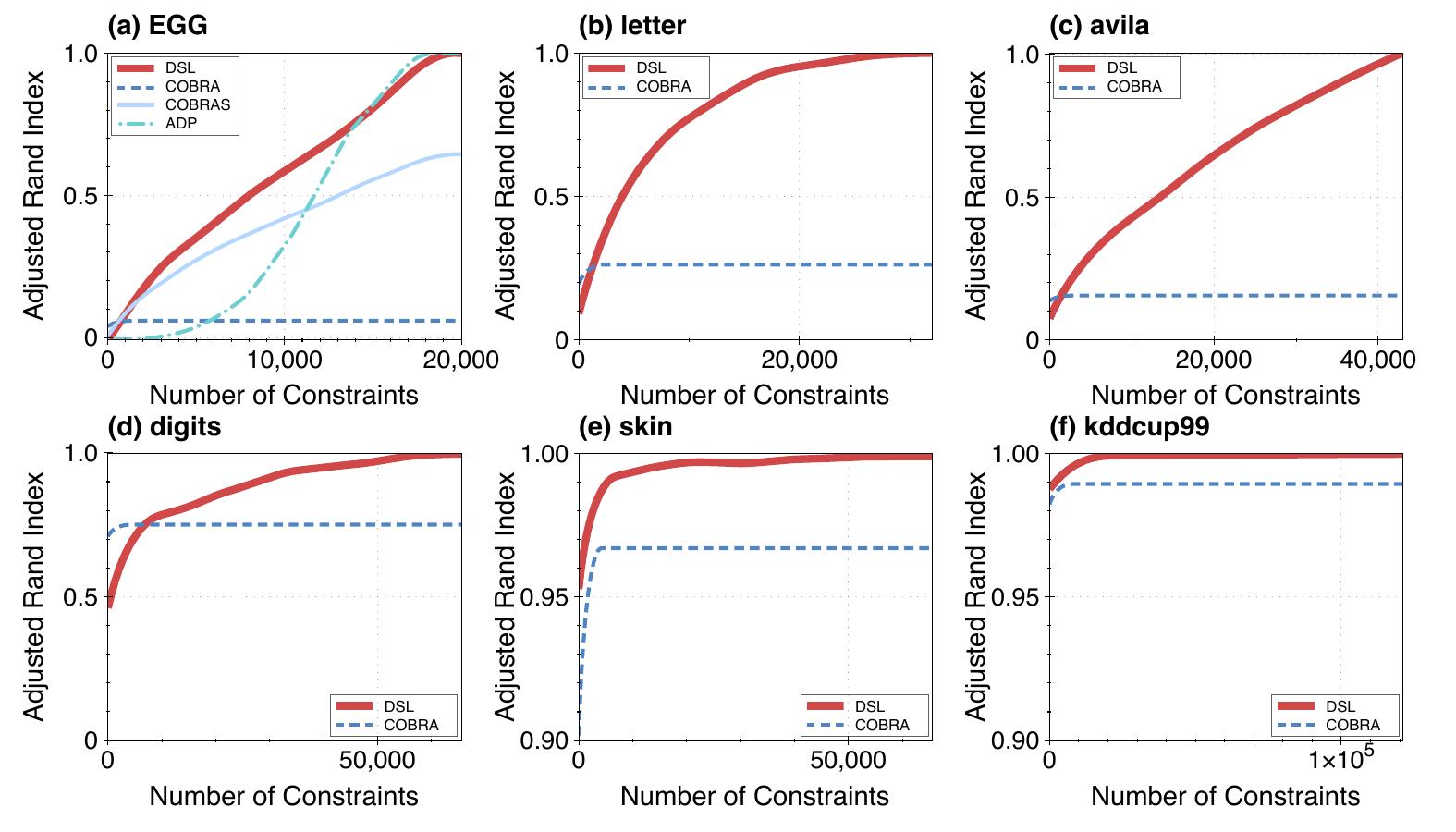}
    \caption{Comparison of ICE Curves on Large-Scale Datasets.}
    \label{fig:overall-data-2}
\end{figure}

\begin{table}[t]
  \centering
  \caption{AUIC@$n$ Score Comparison. The best scores are highlighted. `--' denotes runtimes over 48 hours or exceeding memory capacity on our testbed.}
  \label{tab:ACCI-samll}
  \resizebox{\linewidth}{!}{
  \setlength{\tabcolsep}{5mm}{
    \begin{tabular}{@{}cccccccccc@{}}
      \toprule
      \textbf{Dataset}    & \textbf{DSL}  & \textbf{COBRA} & \textbf{COBRAS} & \textbf{ADP} & \textbf{ADPE}    & \textbf{SPACE}  & \textbf{FFQS} & \textbf{MinMax}\\
      \midrule
      tae & \cellcolor{yellow!25}  \textbf{0.229} & 0.101 & 0.152 & 0.188 & 0.129 & 0.055 & 0.201 & 0.216 \\
      divorce & \cellcolor{yellow!25} \textbf{0.984} & 0.705 & 0.958 & 0.977 & 0.978 & 0.898 & 0.958 & 0.983 \\     
      wine & \cellcolor{yellow!25} \textbf{0.943} & 0.689 & 0.894 & 0.923 & 0.854 & 0.836 & 0.933 & 0.942 \\      
      thyroid & 0.854 & 0.790 & \cellcolor{yellow!25} \textbf{0.867} & 0.189 & 0.200 & 0.698 & 0.838 & 0.864 \\      
      ecoli & \cellcolor{yellow!25} \textbf{0.838} & 0.705 & 0.706 & 0.634 & 0.612 & 0.477 & 0.569 & 0.600 \\      
      musk & \cellcolor{yellow!25} \textbf{0.374} &  0.317 & 0.338 & 0.097 & 0.094 & 0.000 & 0.238 & 0.230 \\      
      led &   0.488 & 0.504 & 0.474 & 0.427 & 0.146 & 0.452 & \cellcolor{yellow!25} \textbf{0.568} & 0.547   \\
      balance & \cellcolor{yellow!25} \textbf{0.585} & 0.504 & 0.543 & 0.279 & 0.550 & 0.111 & 0.454 & 0.498 \\
      breast & \cellcolor{yellow!25} \textbf{0.956} & 0.845 & 0.826 & 0.923 & 0.918 & 0.837 & 0.918 & 0.954 \\
      vehicle & \cellcolor{yellow!25} \textbf{0.416} & 0.333 & 0.326 & 0.301 & 0.246 & 0.084 & 0.198 & 0.207 \\
      banknote & \cellcolor{yellow!25} \textbf{0.993} &  0.942 & 0.973 & 0.237 & 0.970 & 0.001 & 0.518 & 0.502 \\
      segment & \cellcolor{yellow!25} \textbf{0.913} & 0.754 & 0.875 & 0.514 & 0.622 & 0.465 & 0.726 & 0.665 \\  
      EEG &     \cellcolor{yellow!25} \textbf{0.454} & 0.059  & 0.332 & 0.245 & -- & -- & -- & -- \\
      letter & \cellcolor{yellow!25} \textbf{0.801} & 0.261   & -- & -- & -- & -- & -- & -- \\
      avila & \cellcolor{yellow!25} \textbf{0.510}  & 0.155   & -- & -- & -- & -- & -- & -- \\ 
      digits & \cellcolor{yellow!25} \textbf{0.916} & 0.750   & -- & -- & -- & -- & -- & -- \\
      skin & \cellcolor{yellow!25} \textbf{0.999}   & 0.967   & -- & -- & -- & -- & -- & -- \\
      kddcup99 & \cellcolor{yellow!25} \textbf{0.999} & 0.989   & -- & -- & -- & -- & -- & -- \\
      \midrule
      \textbf{MEAN} & \cellcolor{yellow!25} \textbf{0.736} & 0.576 & 0.635 & 0.456 & 0.527 & 0.410 & 0.593 & 0.601 \\
      \bottomrule
    \end{tabular}
    }
 }
\end{table}


Furthermore, we conducted non-parametric statistical tests on the AUIC@$n$ scores to verify whether our algorithm outperforms its counterparts with statistical significance. First, we applied the Friedman test to determine whether there were statistically significant differences among the algorithms. For fairness, we used the first 12 datasets (tae to segment) to avoid bias from missing data in later datasets. The null hypothesis ($H_0$) tested that all methods have equal average performance, and the alternative hypothesis ($H_a$) indicated that at least two methods perform differently. Algorithms were ranked based on their AUIC@$n$ scores, and the Friedman test yielded a test statistic $Q = 41.506$ with a $p$-value of $ 6.472 \times 10^{-7}$, confirming significant differences across the methods ($p < 0.05$). Subsequently, the Nemenyi post-hoc test identified pairwise differences. The CD diagram (Figure \ref{fig:post_hoc}) shows that DSL’s rank (1.42) is significantly better than SPACE, ADP, ADPE, COBRA, and FFQS, confirming its superiority.

\begin{figure}[ht]
    \begin{center}
    \includegraphics[width=0.9\linewidth]{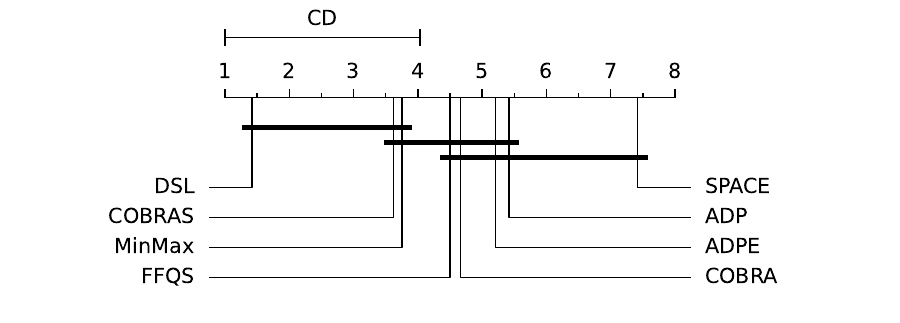}
    \caption{{Critical difference (CD) diagram for the  post-hoc Nemenyi test ($\alpha = 0.05$) applied to AUIC@$n$ scores. A thick horizontal line groups a set of methods that are not significantly different.}}
    \label{fig:post_hoc}
    \end{center}
\end{figure}


\noindent\textbf{Exp. 3: Computational Efficiency}. We verified the efficiency of DSL by measuring response times on synthetic datasets of incremental sizes. Figure \ref{fig:cpu_time}(a) presents the Response Time vs. Data Size for our DSL algorithm and other counterparts. In this figure, the relationship between response time and data size for each algorithm, except for the SPACE algorithm, is fitted using a power-law function, represented as straight lines in a double logarithmic plot. The slopes of these lines indicated in the legend represent the growth rate of the algorithm's response time. Figure \ref{fig:cpu_time}(a) shows that DSL's growth rate in response time was lower than most of the compared algorithms, except for COBRA. Despite COBRA's slightly better computational efficiency, our DSL algorithm surpasses it in clustering quality, achieving 100\% accuracy, a level not reached by COBRA. This advantage underscores the practical value of our algorithm as an interactive data pre-grouping tool in real-world applications.

\begin{figure}[t]
    \centering
    \includegraphics[width=0.9\linewidth]{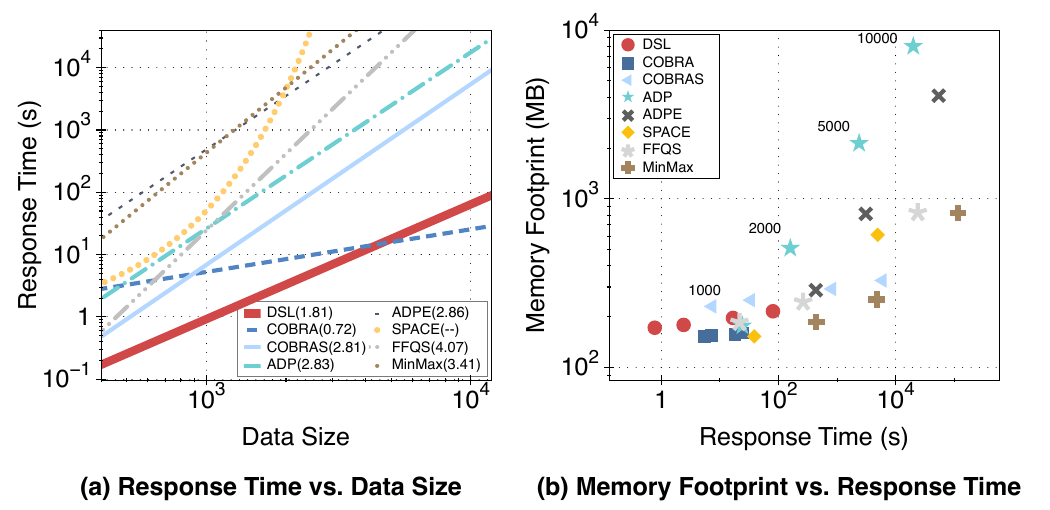}
    \caption{Response Time v.s. Data Size. Note that SPACE is of a non-linear fitting line due to its non-polynomial time-complexity.}
    \label{fig:cpu_time}
\end{figure}

\noindent\textbf{Exp. 4: Scalability.} The scalability of our algorithm is also evidenced in Table \ref{tab:accuracy-2}, focusing on the final six rows pertaining to large-scale datasets. Only our DSL algorithm and COBRA are capable of fully processing these datasets, with DSL invariably achieving higher clustering accuracy. Notably, COBRA's preprocessing mechanism caps its constraint intake, halting updates to clustering outcomes past a certain point and thereby often failing to reach full accuracy, as Table \ref{tab:accuracy-2} indicated.

In addition, we also measured the Memory Footprint against Response Time for varied-sized synthetic datasets, as shown in Figure \ref{fig:cpu_time}(b). Dataset sizes are noted next to the data points. DSL consistently shows superior response times and memory footprints, except for COBRA's slightly lower memory footprints, reinforcing DSL's scalability.

\vspace{-1em}
\section{Application: Face Analysis on Olivetti Dataset}
\label{sec:application}
\vspace{-0.5em}

We further assess DSL's real-world application for face analysis on the Olivetti Face Dataset, demonstrating its compatibility with different distance measures and its operational efficiency in interactive settings, which confirms its efficacy for practical deployment.

\noindent\textbf{Clustering Results Across Distance Metrics.} Here, we analyze the DSL algorithm's robustness using various similarity metrics for clustering within Euclidean space. We introduce four distinct measures for this purpose: the conventional Euclidean and cosine distances applied to grayscale image vectors, the deep feature-based similarity derived from the ResNet50 model, and the wavelet-transform CW-SSIM index. A random distance metric is also introduced as a baseline for perspective. 

As illustrated in Figure \ref{fig:app1-sm}(a), despite some variances among the distance metrics, they all achieve 100\% accuracy after approximately 1500 constraints, demonstrating a general consistency in their performance trends. Notably, the Random distance metric's eventual achievement of 100\% accuracy after just over 7000 constraints reinforces the theoretical predictions that a maximum of $(k+1)n$ constraints is sufficient for fully accurate clustering. This conclusively indicates the robustness of our DSL algorithm to different similarity scaling in real-world applications.

\begin{figure}[t]
    \centering
    \includegraphics[width=0.9\linewidth]{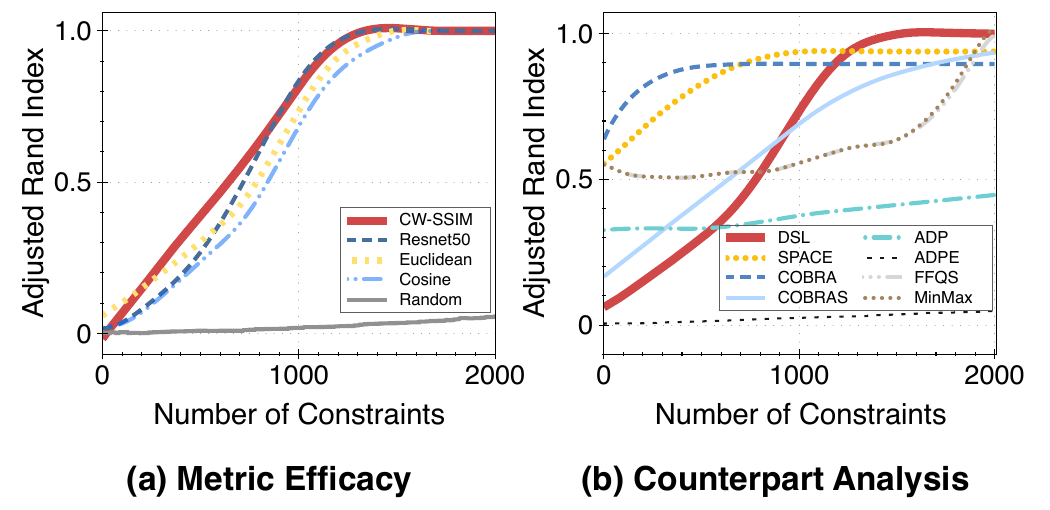}
    \caption{Comparative Analysis on Olivetti Dataset}
    \label{fig:app1-sm}
\end{figure}

\noindent\textbf{Comparative with Benchmark Algorithms.} As shown in Figure \ref{fig:app1-sm}(b), DSL stands out for its efficiency, reaching complete accuracy with the fewest constraints. Although COBRA and SPACE exhibit an initial surge in clustering performance, they do not achieve the 100\% mark. This contrast highlights DSL's advantage in effectively using constraints to achieve complete and consistent clustering accuracy, avoiding the early performance stalls that some algorithms experience. 

These results underscore the practical superiority of our DSL algorithm, especially in contexts where the economy of constraints is crucial. Its efficient and steady performance demonstrates the algorithm's adaptability and reliability for real-world applications. 

\begin{figure}[t]
\centering
    \includegraphics[width=1.0\linewidth]{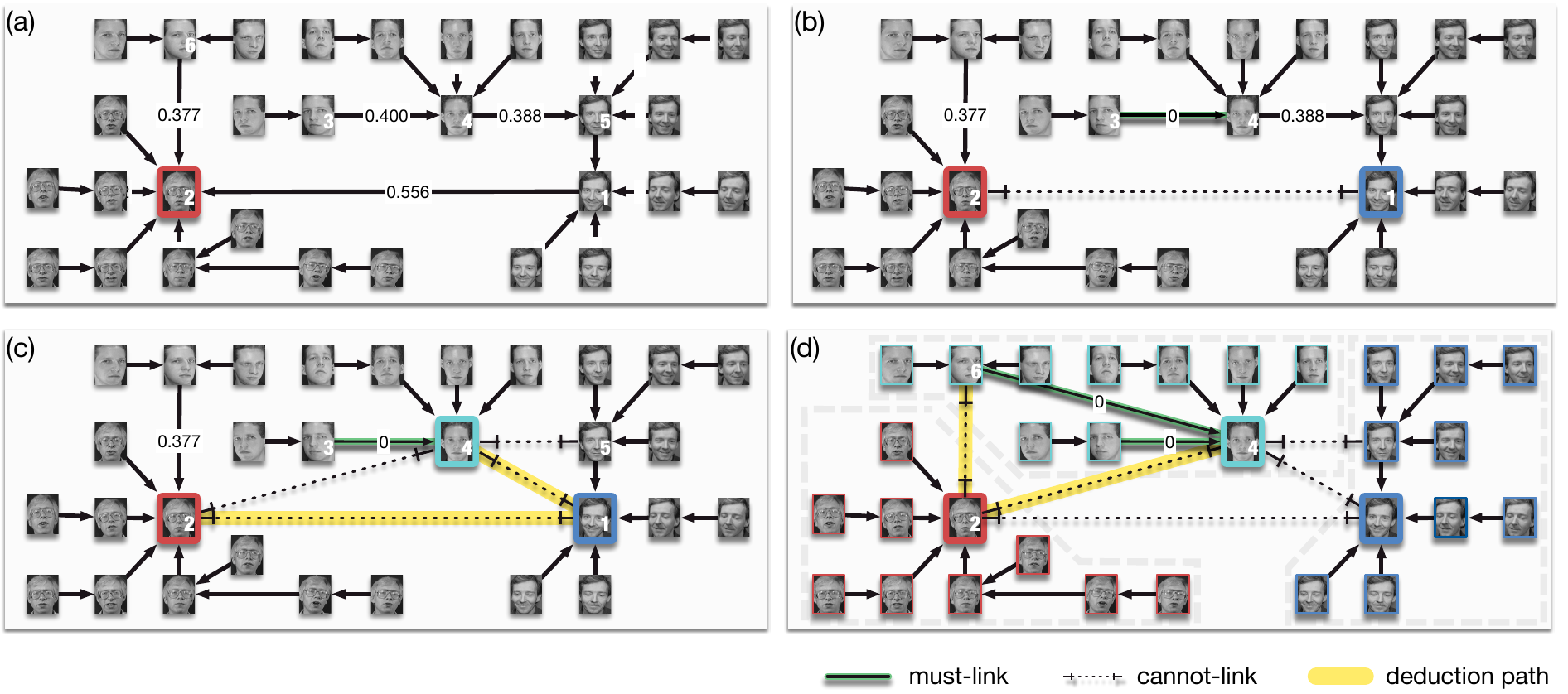}
    \caption{Practical Demonstration on the Olivetti Face Dataset}
    \label{fig:face-sm}
\end{figure}

\noindent\textbf{Sample Set Demonstration.} 
We employ the initial 30 faces from the Olivetti Face Dataset as input. In the \textsc{DSInit} phase, the data skeleton is established, as depicted in Figure \ref{fig:face-sm}(a), where Image 2 serves as the representative. During the \textsc{Recons} phase, DSL refines the edges of four images (Images 1, 3, 4, and 5) and achieves a perfect clustering result, as depicted in Figure \ref{fig:face-sm}(b-d):
(b) The original link of Image 1 is deleted, and Image 1 is formed as a new representative. The original link of Image 3 is maintained while adjusting its distance weight to 0.
(c) The original link of Image 4 is deleted, and Image 4 is formed a new representative. In this process, when judging the relationship between Images 4 and 2, deduction from Equation \ref{equ:equ3} is employed. However, due to two cannot-link constraints along the shortest path (highlighted in yellow), user-provided constraints is needed.
(d) The original link of Image 5 is deleted, and Image 5 connects to the nearest representative, Image 4. In this process, judging the relationship between Images 5 and 4 also requires deduction from Equation \ref{equ:equ3}, yet for the same reason, it necessitates user-provided constraints.
In summary, DSL only relies on 7 constraints to process the 30 images.

\vspace{-1.5em}
\section{Conclusion}
\vspace{-1em}

In this work, we propose the Data Skeleton Learning (DSL) framework, tailored for scalable active clustering. DSL captures the clustering result using a sparsely connected graph (data skeleton), which is refined iteratively by a sparse constraint graph with a unique deduction mechanism, ensuring efficient updates. This approach significantly reduces the need for extensive user-provided constraints while preserving the essential structure of data clusters. Theoretical and empirical evaluations have affirmed its efficiency and scalability. When applied to the real-world face dataset, DSL also exhibits robust performance against various distance metrics, further validating its practicability. This development optimizes the active clustering workflow, integrating human insight with algorithmic precision, essential for processing large-scale data in AI projects. 

While the DSL framework presents several advantages, there are some limitations that need to be addressed in future work. As the scale of datasets increases, the number of user-provided constraints will inevitably rise, potentially exceeding the capacity of a single annotator. Although our algorithm has enhanced the efficiency of annotation, the scalability of human input remains a challenge. Therefore, an important future research direction is the development of collaborative annotation systems that distribute tasks among multiple annotators. This strategy introduces new complexities, such as ensuring the synchronization of global data relationships, managing asynchronous updates, and mitigating discrepancies in interpretations across annotators. Tackling these challenges will be crucial to maintaining both scalability and consistency in the system. Additionally, the current method assumes that user-provided constraints are error-free, which may not always be the case in practice. Errors in constraints can compromise the algorithm’s ability to achieve optimal clustering. A further research priority is to enhance the algorithm’s robustness against such errors, while also focusing on optimizing memory efficiency. These advancements will broaden the applicability of the DSL framework, making it more effective in combining human input with algorithmic processes for large-scale data clustering tasks.



\scriptsize
\vspace{-2em}
\section*{Acknowledgments}
\vspace{-0.5em}
This work is supported by the Sichuan Science and Technology Program [No. 2024NSFSC1464] and the National Natural Science Foundation of China [No. 62172102] and the Sichuan Scientific Innovation Fund [No. 2022JDRC0009] and the Natural Science Starting Project of SWPU [No. 2023QHZ010].

\vspace{-2em}
\section*{CRediT authorship contribution statement}
\vspace{-0.5em}
\textbf{Wen-Bo Xie:} Conceptualization (leading the theoretical development and proof, especially for Theorem 9), Writing - Original Draft \& Editing \& Review, Funding acquisition, Formal analysis. \textbf{Xun Fu:} Methodology, Software, Validation, Formal analysis, Writing - Original Draft \& Editing. \textbf{Bin Chen:} Writing - Editing, Formal analysis. \textbf{Yan-Li Lee:} Writing - Review, Formal analysis, Supervision. \textbf{Tao Deng:} Writing - Editing, Visualization.  \textbf{Tian Zou:} Writing - Editing, Visualization. \textbf{Xin Wang:} Writing - review, Supervision, Funding acquisition, Project administration.  \textbf{Zhen Liu:} Formal analysis, Writing - review. \textbf{Jaideep Srivastava:} Formal analysis, Supervision.

\vspace{-2em}

\bibliographystyle{elsarticle-num-names}

\setlength{\bibsep}{0em}
\bibliography{reference_clean}


\end{document}